\documentclass[twoside,11pt]{article}

\usepackage{blindtext}

\usepackage[preprint]{jmlr2e}

\usepackage{jmlr2e}

\usepackage{lastpage}
\jmlrheading{23}{2022}{1-\pageref{LastPage}}{1/21; Revised 5/22}{9/22}{21-0000}{Ethan Blaser, Chuanhao Li, and Hongning Wang}

\ShortHeadings{Federated Linear Contextual Bandits with Heterogeneous Clients}{Blaser, Li, and Wang}
\firstpageno{1}

\usepackage{algorithm}
\usepackage{algorithmic}

\usepackage{booktabs}       
\usepackage{nicefrac}       
\usepackage{microtype}      
\usepackage{xcolor}         
\usepackage{latexsym}
\usepackage{url}
\usepackage{array}
\usepackage{graphicx}
\usepackage{booktabs}
\usepackage{amsmath}
\usepackage{dsfont}
\usepackage{subcaption}

\usepackage{amssymb}
\usepackage{mathrsfs}
\DeclareMathOperator*{\argmax}{arg\,max}

\newcommand{\pluseq}{\mathrel{+}=}

\def \cC {\mathcal{C}}
\def \cF {\mathcal{F}}
\def \cS {\mathcal{S}}

\def \cH {\mathcal{H}}
\def \cA {\mathcal{A}}

\def \bx {\mathbf{x}}
\def \bX {\mathbf{X}}
\def \bY {\mathbf{Y}}
\def \by {\mathbf{y}}

\def \cA {\mathcal{A}}

\def \bI {\mathbf{I}}
\def \bR {\mathbb{R}}

\newtheorem{assumption}{Assumption}

\newcommand{\model}{\textsc{HetoFedBandit}}
\newcommand{\modelenhance}{\textsc{HetoFedBandit-E}}

\begin{document}

\title{Federated Linear Contextual Bandits with Heterogeneous Clients}

\author{\name Ethan Blaser \email ehb2bf@virginia.edu \\
       \addr Department of Computer Science\\
       University of Virginia\\
       Charlottesville, VA 22904, USA
       \AND
       \name Chuanhao Li \email cl5ev@virginia.edu \\
       \addr Department of Computer Science\\
       University of Virginia\\
       Charlottesville, VA 22904, USA
       \AND
       \name Hongning Wang \email hw5x@virginia.edu \\
       \addr Department of Computer Science\\
       University of Virginia\\
       Charlottesville, VA 22904, USA}

\editor{My editor}

\maketitle

\begin{abstract}
The demand for collaborative and private bandit learning across multiple agents is surging due to the growing quantity of data generated from distributed systems. Federated bandit learning has emerged as a promising framework for private, efficient, and decentralized online learning. However, almost all previous works rely on strong assumptions of client homogeneity, i.e., all participating clients shall share the same bandit model; otherwise, they all would suffer linear regret. This greatly restricts the application of federated bandit learning in practice. In this work, we introduce a new approach for federated bandits for heterogeneous clients, which clusters clients for collaborative bandit learning under the federated learning setting. Our proposed algorithm achieves non-trivial sub-linear regret and communication cost for all clients, subject to the communication protocol under federated learning that at anytime only one model can be shared by the server.
\end{abstract}


\section{Introduction}\label{sec:intro}

Bandit learning algorithms \citep{auer2002finite,chapelle2011empirical,lihongcontextual,abbasi2011improved} have become a reference solution to the problems of online decision optimization in a wide variety of applications, including recommender systems \citep{lihongcontextual}, clinical trials \citep{durand2018contextual}, and display advertising \citep{li2010exploitation}. Typically, these algorithms are operated by a centralized server; but due to the growing quantity of data generated from distributed systems, there is a surge in demand for private, efficient, and decentralized bandit learning across multiple clients. Federated bandit learning has emerged as a promising solution framework,
where multiple clients collaborate to minimize their cumulative regret under the coordination of a single central server \citep{dislinucb,dubey2020,fedmab,Liasynch,he2022a}. The server's role is limited to facilitating joint model estimation across clients, without having access to any clients' arm pulling or reward history. 

Although federated bandit learning has gained increasing interest from the research community, most existing approaches necessitate that all clients share the same underlying bandit model in order to achieve near-optimal sub-linear regret for a population of clients. This strong homogeneity assumption distills federated bandit learning to a joint estimation of a single global model across clients, subject to the federated learning communication protocol \citep{bonawitz2019towards,kairouz2021advances}. 
However, in reality, clients can have diverse objectives, resulting in different optimal policies. Imposing a single global model on a heterogeneous client population can easily cost every client linear regret \citep{hossain2021fair}. Consequently, rational clients should choose not to participate in such a federated learning system, as they cannot determine if other participating clients share the same bandit model with them beforehand, and they can already achieve sub-linear regret independently (albeit inferior to the regret obtained when all clients genuinely share the same bandit model). This seriously impedes the practical application of existing federated bandit learning solutions.

In a parallel line of bandit research, studies in collaborative bandits aim to improve bandit learning in heterogeneous environments by facilitating collective model estimation among different clients. For example, clustered bandit algorithms group similar clients and use a shared bandit model for clients within the same group
\citep{gentile2014online,li2016collaborative,gentile2017context,cesa2013gang,wu2016contextual}. When the relatedness among clients are provided, such as through an affinity graph, joint policy learning can be performed analytically \citep{cesa2013gang,wu2016contextual}. However, most of the existing collaborative bandit learning algorithms operate under a centralized setting, in which data from all clients is assumed to be directly accessible by a central server. As a result, these methods cannot address the demand for privacy and communication efficiency in online learning for distributed systems. Significant efforts are required to adapt these algorithms to distributed settings \cite{mahadik2020fast}. 

In this paper, we introduce a novel approach for federated bandit learning among heterogeneous clients, extending collaborative bandit learning to the standard federated learning setting. The goal is to ensure that every participating client achieves regret reduction compared to their independent learning, thereby motivating all clients to participate. As the first work of this kind, we focus on estimating a linear contextual bandit model \citep{lihongcontextual,abbasi2011improved} for each client, which is also the most commonly employed model in federated bandits.  
Not surprisingly, regret reduction in a population of heterogeneous clients can be realized by clustering the clients, where collective model estimation is only performed within each cluster. 
But the key challenges lie in the communication protocol in federated learning. First, the server lacks real-time access to each client's data, resulting in delayed inferences of client clusters. Second, the server can only estimate and broadcast one global model at a time \citep{chaoyanghe2020fedml,openfl_citation}. This can cause communication congestion and delay the model updates. Both of them cost regret.

To address these challenges we develop a two-stage federated clustered bandit algorithm. In the first stage, all clients perform pure exploration to prepare a non-parametric clustering of clients based on the statistical homogeneity test \citep{lidyclu}. Then in the second stage, a first-in-first-out queue is maintained on the server side to facilitate event-triggered communication \citep{dislinucb} at the cluster level. 
We rigorously establish the upper bounds of cumulative regret and communication cost for this algorithm.
Then, we empirically enhance the algorithm by allowing dynamic re-clustering of clients in the second stage and employ a priority queue to improve regret. We conduct comprehensive empirical comparisons of the newly proposed federated bandit algorithm against a set of representative baselines to demonstrate the effectiveness of our proposed framework.
\section{Related Work} \label{sec:relatedWork}
Our work is closely related to studies in federated bandit learning and collaborative bandits. In this section, we discuss the most representative solutions in each area and highlight the relationships between them and our work.

\noindent\textbf{Federated Linear Contextual Bandits:} 
There have been several works that study the federated linear contextual bandit setting, where multiple clients work collaboratively to minimize their cumulative regret with the coordination of a single central server \citep{dislinucb,Liasynch,huangfedcontextualbandits}. \cite{dislinucb} introduced DisLinUCB, where a set of homogeneous clients, each with the same linear bandit parameter, conduct joint model estimation through sharing sufficient statistics with a central server. \cite{Liasynch} and \cite{he2022a} extended this setting by introducing an event-triggered asynchronous communication framework to achieve sub-linear communication cost as well as sub-linear regret in a homogeneous environment. Additionally, \cite{dubey2020differentially} considers deferentially private federated contextual bandits in peer-to-peer communication networks. Fed-PE, proposed in \citep{huangfedcontextualbandits}, is a federated phase-based elimination algorithm for linear contextual bandits that handles both homogeneous and heterogeneous settings. However, in their setting, the client is trying to learn the fixed context vectors associated with each arm as opposed to the linear reward parameter (which is known in their setting). With the exception of Fed-PE, which utilizes a different bandit formulation altogether, all of these prior works rely on strong assumptions of client homogeneity, while our work seeks to extend federated linear contextual bandit learning to a heterogeneous environment. 

\noindent\textbf{Collaborative Bandits:}
Collaborative bandits seek to leverage similarities between heterogeneous clients to improve bandit learning. Clustered bandit algorithms are one example, where similar clients are grouped together, and a shared bandit model is used for all clients in the same group\citep{gentile2014online,li2016collaborative,gentile2017context,cesa2013gang,wu2016contextual}. \cite{gentile2014online} assumed that observations from different clients in the same cluster are associated with the same underlying bandit parameter. \cite{gentile2017context} further studied context-dependent clustering of clients, grouping clients based on their similarity along their bandit parameter's projection onto each context vector. \cite{lidyclu} unified non-stationary and clustered bandit by allowing for a time varying bandit parameter for each client, which requires online estimation of the dynamic cluster structure at each time.  Other works leverage explicit inter-client and inter-arm relational structures, such as social networks \citep{buccapatnam2013multi,cesa2013gang,wu2016contextual, hongbayesian, kvetonsideobservations2012, mannorsideobservations} to facilitate collaboration. However, most existing collaborative bandit solutions are designed under a centralized setting, where all clients' observation data is readily available at a central server. \cite{liufederatedonlineclusteringofbandits} and \cite{korda2016_peer_clustering} consider online cluster estimation in a distributed setting. However, their federated learning architectures do not align with the standard federated learning architecture and real world implementations where a single central server broadcasts a single global model at each timestep \citep{McMahanMRA16,chaoyanghe2020fedml,openfl_citation}. Specifically, \cite{liufederatedonlineclusteringofbandits} utilizes a hierarchical server configuration that is distinct from the standard single-server FL setup. On the other hand, \cite{korda2016_peer_clustering} is based on a peer-to-peer (P2P) communication network, which stands in contrast to the centralized communication model and also overlooks the potential communication costs associated with such a decentralized approach.
\section{Methodology}
\label{sec:method}
In this section, we begin by outlining the problem setting investigated in this work. Then we present our two-stage federated clustered bandit algorithm designed to serve a population of heterogeneous clients under the standard communication setup in federated learning. We provide theoretical analysis of the upper regret bound for our developed solution. Lastly, we introduce a set of improvements to our proposed algorithm, including dynamic re-clustering of clients using an adaptive clustering criterion, and the implementation of a priority queue to enhance online performance, both of which were found empirically effective.

\subsection{Problem Setting} \label{subsec:setting}

A federated bandit learning system consists of two components: 1) $N$ clients, which take actions and get reward feedback from their environment (e.g., edge devices in a recommendation system interacting with end users) and 2) a central server coordinating client communication for collaborative model estimation. In each time step $t=1,2,...T$, each client $i \in {N}$ chooses an action $x_{t,i}$ from its action set $\mathcal{A}_{t,i} = \{{x}_{t,1}, {x}_{t,2}...,x_{t,K}\}$, where $x\in \bR^d$. Adhering to the standard linear reward assumption from \citep{li2010exploitation}, the corresponding reward received by client $i$ is $y_{t,i} = \langle \theta_i^{*},{x}_{t,i}\rangle + \eta_t$, where noise $\eta_t$ comes from a $\sigma^2$ sub-Gaussian distribution, and $\theta_{i}^*$ is the true linear reward parameter for client $i$.  Without loss of generality, we assume $\|x\|_2 \leq 1$ and $\|\theta_{i}^*\| \leq 1$.

The learning system interacts with the environment for $T$ rounds, aiming to minimize the cumulative pseudo-regret $R_T = \sum_{t = 0}^T \sum_{i=0}^N \max_{x\in \cA_{t,i}} \langle \theta_i^{*},x\rangle - \langle \theta_i^{*},{x}_{t,i}\rangle$.

Following the federated learning setting, we assume a star-shaped communication network, where the clients cannot directly communicate among themselves. Instead, they must share the learning algorithm's parameters (e.g., gradients, model weights, or sufficient statistics) through the central server. To preserve data-privacy, raw observations collected by each client $(x_{t,i},y_{t,i})$ are stored locally and will not be shared with the  server.
At every timestep $t=1,...T$, the central server is capable of using the shared learning algorithm to update and broadcast one model to the selected clients. The communication cost is defined as the amount of sufficient statistics communicated across the learning system  over the entire time-horizon.


Unlike existing federated bandit works \citep{dislinucb, Liasynch, he2022a} which assume homogeneous clients, we adopt the standard clustered bandit setting to model a heterogeneous learning environment. Without an underlying cluster structure in the environment, collaboration between clients would be infeasible. Therefore, we assume that clients sharing similar reward models form clusters, collectively represented as $\cC = \{C_1,C_2,...,C_M \}$. The composition and quantity of these clusters, are unknown to the system, necessitating on-the-fly inference. Consistent with prevalent clustered bandit practices \citep{gentile2014online,gentile2017context,liufederatedonlineclusteringofbandits}, we use unknown environmental parameters $\epsilon$ and $\gamma$ to delineate the ground-truth cluster structures:

\begin{assumption}[Proximity within clusters] \label{ass:proximity} For any two clients $i,j$ within a particular cluster $C_k \in \cC$, $\|\theta^*_i - \theta^*_j\| \leq \epsilon$ where $\epsilon = 1/(N\sqrt{T})$.
\end{assumption}

\begin{assumption}[Separateness among clusters] \label{ass:separation} For any two clusters $C_k,C_l \in \cC, \; \forall i\in C_k,\, j\in C_l$, $\|\theta^*_i - \theta^*_j\| \geq \gamma \geq 0 \;$ \citep{gentile2014online,gentile2017context,lidyclu,liufederatedonlineclusteringofbandits}.
\end{assumption}

 Contrary to previous clustered bandit assumptions of identical reward models within a cluster, our Assumption \ref{ass:proximity} offers more flexibility. It enables similar clients (represented by $\epsilon$) to collaborate, amplifying the system's collaborative benefit. We also adopt a standard context regularity assumption found in clustered bandits.

\begin{assumption}[Context regularity]\label{as:contextreg}
At each time $t$, $\forall i\in \{N\}$ arm set $\cA_{t,i}$ is generated i.i.d. from a sub-Gaussian random vector $x_{t,i} \in \bR^{d}$, such that $\mathbb{E}[x_{t,i}x_{t,i}^{\top}]$ is full-rank with minimum eigenvalue $\lambda_c>0$ \citep{gentile2014online,gentile2017context,li2019improved}.
\end{assumption}

Notably, our context regularity Assumption \ref{as:contextreg} is weaker than those in \citep{gentile2014online,gentile2017context,li2019improved}. Ours only requires the lower bound on the minimum eigenvalue of $\mathbb{E}[x_{t,i}x_{t,i}^{\top}]$, while others require the imposition of a variance condition on the stochastic process generating $x_{t,i}$. 


To facilitate our later discussions, we use $\cH_{t,i}=\left\{(x_{\tau,i},y_{\tau,i})\right\}_{\tau=1}^{t}$ to represent the set of $t$ observations from client $i$. $(\bX_{i},\by_{i})$ denote design matrices and feedback vectors of $\cH_{t,i}$ where each row of $\bX$ is the context vector of an arm and the corresponding element in $\by$ is the observed reward for this arm. Note that $\mathbf{X}_{j}$ only contains the observations made by client $j$ and does not include aggregated observations from other clients in the cluster. We also define the weighted norm of a vector $x\in \bR^d$ as $\|x\|_A =\sqrt{x^\top A x}$, where $A\in \bR^{d\times d}$ is a positive definite matrix.

\subsection{Algorithm: \model} \label{subsec:alg}
In this section, we present our two-stage federated clustered bandit algorithm. As discussed in Section \ref{sec:intro}, there are two primary challenges associated with extending clustered bandit learning to the federated learning setting. The first challenge is to identify the subsets of heterogeneous clients that can benefit from collaboration among themselves. 
To achieve this, in the first stage of our algorithm, all clients conduct random exploration ahead of a non-parametric clustering of clients based on the statistical homogeneity test \citep{lidyclu}. The second challenge arises from the communication network setting in federated learning framework, which allows only one model to be broadcast at each time step \citep{chaoyanghe2020fedml,openfl_citation}. 
To accommodate this constraint, a first-in-first-out queue is utilized on the server side, enabling event-triggered collaboration \citep{dislinucb} at the cluster level during our algorithm's second phase. 
We provide an overview of the key components of our algorithm, with the full details available in Algorithm \ref{alg:simplified}.

\begin{figure}[h]
\includegraphics[height= 8cm, width=\columnwidth]{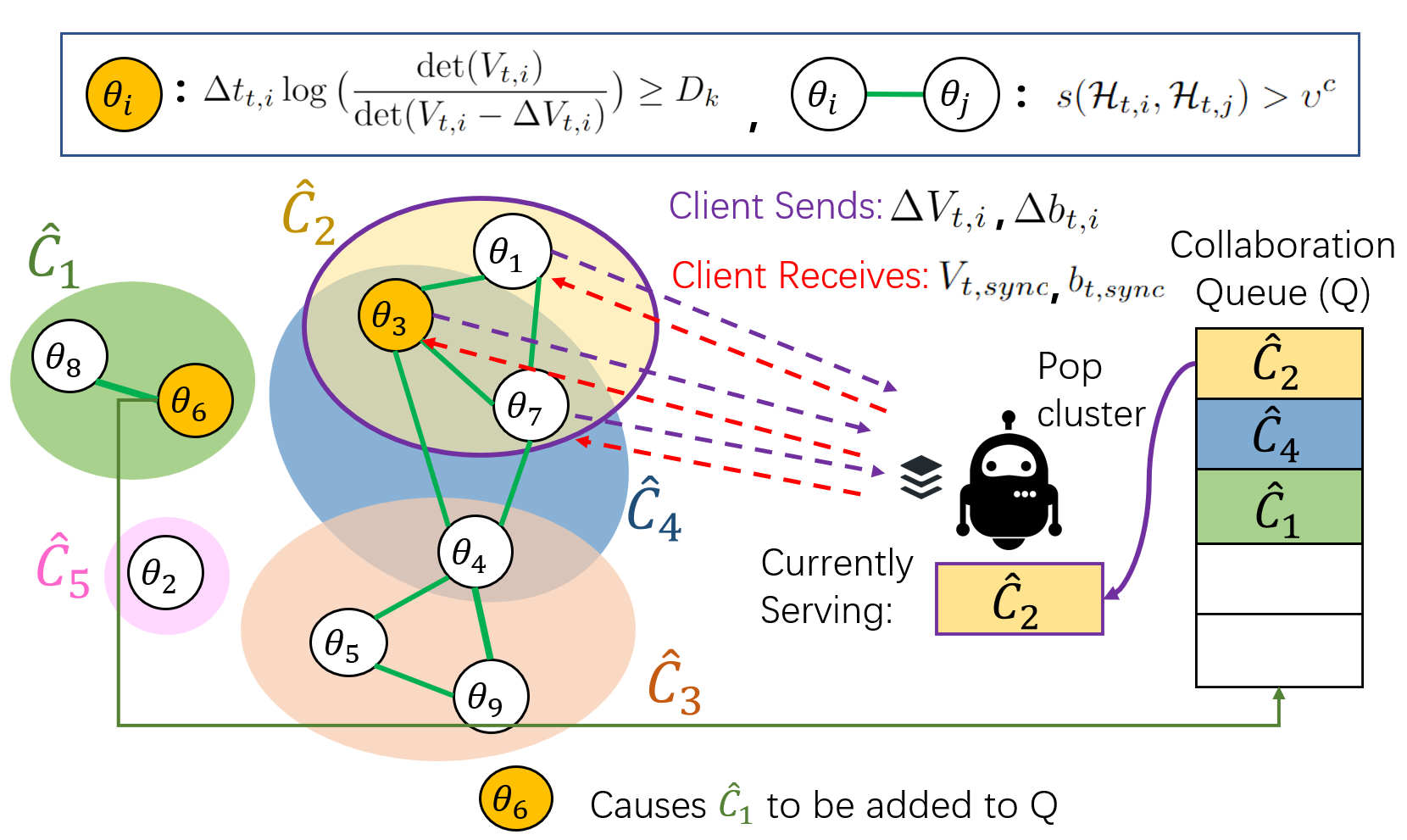}
 \vspace{-3mm}
\caption{Execution of \model \, after pure exploration phase $T_0$. Each client $i \in N$ is represented by a node in the client graph $\mathcal{G}$ on the left hand side. Edges between clients indicate potential collaborators, as defined by the homogeneity test. The colored ellipsoids represent the estimated clusters $\hat{\cC} = \{\hat{C}_1,...,\hat{C}_5\}$, which are the maximal cliques of $\mathcal{G}$. Clients exceeding their communication threshold are highlighted in orange. Currently, client $\theta_6$ has exceeded the communication threshold $D_1$ for cluster $\hat{C}_1$, which causes cluster $\hat{C}_1$ to be added to the queue. The server pops cluster $\hat{C}_2$ from the queue and facilitates collaboration among $\{\theta_1, \theta_3, \theta_7\}$. In the next timestep, the server will serve cluster $\hat{C}_4$, queued for removal.}
 \vspace{-4mm}
\label{fig:algo}
\end{figure}

\paragraph{Pure Exploration Phase} Under our relaxed context regularity assumption, we execute a short exploration phase of length $T_0$ to guarantee the accuracy of our homogeneity test. Our discussion on the choice of $T_0$ is deferred to Section \ref{subsec:proof}. Although our derived theoretical value for $T_0$ depends on an unknown environmental parameter $\gamma$, in practice, $T_0$ can be tuned as a hyperparameter.

During this exploration stage, for each $t\in \{0...T_0\}$, every client $i\in [N]$ selects an action $x_{t,i}$ by uniformly sampling from $\cA_{t,i}$ in parallel. After receiving reward $y_{t,i}$, each client updates their local sufficient statistics $V_{t,i}=V_{t-1,i}+{x}_{t,i} {x}_{t,i}^\top$ and $b_{t,i} =b_{t-1,i}+x_{t,i} y_{t,i}$. Upon completion of $T_0$ rounds of pure exploration, each client then shares its sufficient statistics ($V_{T_0,i}$, $b_{T_0,i}$) to the central server. We present the details of the exploration phase in Algorithm 1.

\begin{algorithm}[h]
    \caption{Pure Exploration Phase}\label{alg:exploration}
  \begin{algorithmic}[1]
    \FOR{$t=1,2,...,T_0$}
        \FOR{Agent $i \in N$}
            \STATE Choose arm $x_{t,i} \in \cA_{i,t}$ uniformly at random and observe reward $y_{t,i}$
            \STATE Update agent $i$: $\cH_{t,i}=\cH_{t-1,i}\cup (x_{t,i},y_{t,i})$, $V_{t,i}\pluseq x_{t,i}x_{t,i}^{\top}$, $b_{t,i}\pluseq x_{t,i}y_{t,i}$,
            \STATE $\Delta V_{t,i}\pluseq x_{t,i}x_{t,i}^{\top}$, $\Delta b_{t,i}\pluseq x_{t,i}y_{t,i}, \Delta t_{0,i}\pluseq 1$ 
        \ENDFOR
    \ENDFOR
  \end{algorithmic}
\end{algorithm}
\paragraph{Cluster Estimation}
The key challenge in online clustering of bandits is to measure the similarity between different bandit models. Previous works identify whether a set of clients share exactly the same underlying reward model; while we cluster similar clients (as defined by $\epsilon$) to widen the radius of beneficial collaboration. We realize this by testing whether $\|\theta^*_{1}-\theta^*_{2}\|\leq \epsilon$ via the homogeneity test introduced in \citep{lidyclu}.

Specifically, we utilize a $\chi^{2}$ test of homogeneity, where the test statistic $s(\cH_{t,1},\cH_{t,2})$ follows the non-central $\chi^{2}$-distribution \citep{chow1960tests,cantrell1991interpretation}. The test determines whether the parameters of linear regression models associated with two datasets are similar, assuming equal variance. Since $\theta_{1}^*$ and $\theta_{2}^*$ are unobservable, the test utilizes the maximum likelihood estimator (MLE) for $\theta$ on a dataset $\cH$, which we denote ${\vartheta}=(\bX^{\top}\bX)^{-}\bX^{\top}\by$, where $\left(\cdot\right)^{-}$ denotes the generalized matrix inverse:
\begin{align}
s(\cH_{t,1},\cH_{t,2})&=\frac{||\bX_{1}({\vartheta}_{1}-{\vartheta}_{1,2})||^{2}\!+\!||\bX_{2}({\vartheta}_{2}-{\vartheta}_{1,2})||^{2}}{\sigma^{2}} 
\label{eq:homo_test}
\end{align}
where ${\vartheta}_{1,2}$ denotes the estimator using data from $\cH_{t,1}$ and $\cH_{t,2}$.

When $s(\cH_{t,1},\cH_{t,2})$ exceeds a chosen threshold $\upsilon^c$, it indicates a deviation between the combined estimator and the individual estimators on the two datasets. Thus, we conclude $\|\theta_{1}^* - \theta_{2}^*\|> \epsilon$;
otherwise, we conclude $\|\theta_{1}^* - \theta_{2}^*\|\leq \epsilon$. 
Therefore, to determine the sets of clients that can collaborate, the central server performs this pairwise homogeneity test among each pair of clients. If two clients $i,j$ satisfy the homogeneity test $s(\mathcal{H}_{T_0,i},\mathcal{H}_{T_0,j})\geq \upsilon^c$, then we add an undirected edge between them in a client graph $\mathcal{G}$ indicating they benefit from mutual collaboration.

Next, our algorithm uses $\mathcal{G}$ to determine the clusters of clients that can benefit from collaboration. Because our algorithm allows collaboration between non-identical clients, we must ensure every client within a cluster is sufficiently similar to every other; otherwise, linear regret can be caused by the incompatible model sharing. 
For this purpose, we require each of our estimated clusters be a maximal clique of $\mathcal{G}$ (line 5 in Algorithm \ref{alg:clustering}).
We denote the set of resulting clusters as $\hat \cC = \{\hat{C}_1, \hat{C}_2, ... \hat{C}_{\hat{M}}\}$ 
and the set of cluster indices in $\cC$ that client $i$ belongs to as $\mathcal{K}_i$. 

Using the maximal cliques of $\mathcal{G}$ as the cluster estimates introduces a unique challenge that affects our subsequent algorithmic design: the estimated clusters may not be disjoint. In Figure \ref{fig:algo}, we can see that client 3, represented by $\theta_3$, is a member of two clusters, $\hat{C}_2$ and $\hat{C}_4$. 
Therefore, $\theta_3$ will receive shared model updates from clients $\{\theta_1,\theta_7,\theta_4\}$. However, the absence of an edge between $\theta_1$ and $\theta_4$ implies that simultaneous collaboration between $\{\theta_3, \theta_1\}$ and $\{\theta_3, \theta_4\}$ is not allowed by our algorithm. As a result, when $\theta_3$ is collaborating with $\theta_1$, it should only share its local data, excluding what it has received from the server.
Later we describe our queue-based sequential approach to resolve this.
\begin{algorithm}[h]
    \caption{Cluster Estimation}\label{alg:clustering}
  \begin{algorithmic}[1]
        \FOR{$(i,j) \in N$}
            \STATE\textbf{if} {$s(\cH_{T_0,i},\cH_{T_0,j}) \leq \upsilon^{c}$} \textbf{then} add edge $e(i,j)$ to $\mathcal{G}$
        \ENDFOR
        \STATE $\hat{\cC} = \{\hat{C}_1, \hat{C}_2, ... \hat{C}_{\hat{M}}\}=$ maximal\_cliques($\mathcal{G}$)
        \STATE Set $\mathcal{K}_{i}$ = $\{k: i\in \hat{C_k}\}$ for each client $i$
  \end{algorithmic}

\end{algorithm}

\paragraph{Optimistic Learning Phase}
Upon identifying client clusters suitable for collaboration, we proceed to the optimistic learning phase of our algorithm. Here, clients optimistically choose arms, utilizing the collaboration with other similar clients to enhance their local model estimates. 
At each time step $t\in \{T_0...T\}$, each client $i\in [N]$ optimistically selects an arm $x_{t,i} \in \cA_{t,i}$ using the UCB strategy based on its sufficient statistics $\{V_{t,i},b_{t,i}\}$:
\begin{equation}
    \label{eq:UCB}
    x_{t}=\argmax_{x \in \cA_{t,i}}{x^{\top}\hat{\theta}_{t-1,i}+\text{CB}_{t-1,i}(x)}
\end{equation}
where $\hat{\theta}_{t-1,i}= \overline{V}_{t-1,i}^{-1}b_{t-1,i}$ is the ridge regression estimator with regularization parameter $\lambda$; $\overline{V}_{t-1,i}=V_{t-1,i}+\lambda I$; and the confidence bound of reward estimation for arm $x$ is $\text{CB}_{t-1,i}(x)=\alpha_{t-1,i}\|x\|_{\overline{V}_{t-1,i}^{-1}}$, where 
$\alpha_{t-1,i}=\sigma \sqrt{2 \log \biggr( \frac{\det(\overline{V}_{t-1,i})^{1/2}}{ \delta \det(\lambda I)^{1/2}}\biggr)} + \sqrt{\lambda}$. Note that $V_{t,i}$ is formulated using data locally collected by client $i$ in conjunction with data from the clients with whom client $i$ has previously collaborated. After client $i$ observes reward $y_{t,i}$, it updates its local sufficient statistics to improve the reward estimates in future rounds. 

\paragraph{Communication Protocol} 


Our algorithm integrates the event-triggered communication protocol from \cite{dislinucb} to efficiently balance communication and regret minimization within clusters. It uses delayed communication, where clients store observations and rewards in a local buffer $\Delta V_{t,i}$ and $\Delta b_{t,i}$. Clients request server collaboration when the informativeness of the stored updates surpass a certain threshold. Specifically, If $\Delta t_{t,i} \log(\det(V_{t,i})/\det(V_{t,i}-\Delta V_{t,i}))\geq D_k$ for any $k \in \mathcal{K}_i$, with $D_k$ as the communication threshold for the estimated cluster $\hat{C}_k$, the client sends a collaboration request for $\hat{C}_k$.

Multiple clients across different clusters can trigger simultaneous communication requests, and single clients can request collaboration for multiple clusters because the estimated clusters are not disjoint. In such cases, the central server uses a first-in-first-out queue (FIFO) $Q$ to manage the clusters needing collaboration one-by-one. At each timestep $t\in \{T_0...T\}$, it serves one cluster from the queue, ensuring no inter-cluster data contamination by computing $V_{t,sync}$ and $b_{t,sync}$ using only the clients' upload buffers. Despite the single
global-model restriction in federated learning, our algorithm still helps multiple groups of similar clients in a pseudo round-robin manner. As a result, a cluster of clients can resume engaging with the environment without being hindered by the server's processing time for unrelated clusters that don't offer collaborative advantage.

This queuing strategy enhances the system's efficiency, allowing clusters to re-engage with the environment without idling for the server's processing of all other clusters. However, before computing and sharing $\{V_{t,sync}, b_{t,sync}\}$ for collaboration, our algorithm mandates the complete upload of local buffers from every client in that cluster. This signifies that our algorithm employs asynchronous communication at the cluster level but still requires synchronous communication among clients within the same cluster. In practical distributed systems, clients often exhibit variable response times and occasional unavailability. Adapting our algorithm to support asynchronous communication at the individual client level, such that they can collaborate without awaiting updates from all other clients within the cluster, remains an important open research question.

\begin{algorithm}[h]
    \caption{\model}\label{alg:simplified}
  \begin{algorithmic}[1]
    \STATE \textbf{Input:} $T$, $\delta \in (0,1)$, exploration length $T_0$, $\lambda>0$, neighbor identification $\upsilon^c$
    \STATE \textbf{Initialization:} \textbf{Clients:} $\forall i \in N$: $V_{0,i}=\textbf{0}_{d \times d}, b_{0,i}=\textbf{0}_{d}, \cH_{0,i}=\emptyset,\Delta V_{0,i}=\textbf{0}_{d \times d}, \Delta b_{0,i}=\textbf{0}_{d},\Delta t_{i,0}=0$, $\mathcal{K}_{i} = \emptyset$ ; \textbf{Server: }Client graph $\mathcal{G}$ with $N$ nodes, FIFO queue $Q$;
    \STATE Pure Exploration Phase (Algorithm \ref{alg:exploration})
    \STATE Cluster Estimation (Algorithm \ref{alg:clustering})
    \STATE Cluster communication thresholds $\mathcal{D}=[D_1,...,D_{\hat{M}}]$ where $D_k = (T \log |\hat{C}_k|T) / (d |\hat{C}_k|)$
    \FOR{$t=T_0+1,...,T$}
        \FOR{Client $i \in N$}
            \STATE Choose arm $x_{t,i} \in \cA_{t,i}$ by Eq. \ref{eq:UCB} observe reward $y_{t,i}$
            \STATE Update client $i$: $\cH_{t,i}=\cH_{t-1,i}\cup (x_{t,i},y_{t,i})$, $V_{t,i}\pluseq x_{t,i}x_{t,i}^{\top}$, $b_{t,i}\pluseq x_{t,i}y_{t,i}$, \STATE $\Delta V_{t,i}\pluseq x_{t,i}x_{t,i}^{\top}$, $\Delta b_{t,i}\pluseq x_{t,i}y_{t,i}, \Delta t_{t,i}\pluseq 1$ 
            \FOR{$k \in \mathcal{K}_i$}
                \IF{$\Delta t_{t,i} \log(\det(V_{t,i})/\det(V_{t,i}-\Delta V_{t,i}))\geq D_k$}
                    \STATE Collaboration Request: Server adds $\hat{C}_k$ to $Q$
                \ENDIF
            \ENDFOR
        \ENDFOR
        \IF{$Q$ is non-empty}
            \STATE Server pops $\hat{C_k}$ from $Q$
            \STATE Every client $i\in\hat{C_k}$ sends  $\Delta V_{t,j}$, $\Delta b_{t,j}$ to server
            \STATE Each client in $\hat{C}_k$ receives $V_{t,sync} = \sum_{j \in \hat{C_k}}\Delta V_{t,j}$, $b_{t,sync} = \sum_{j \in \hat{C_k}}\Delta b_{t,j}$ from the server
            \STATE Local client updates: $V_{t,i}\pluseq V_{t,sync} - \Delta V_{t,i}$, $b_{t,i}\pluseq b_{t,sync} - \Delta b_{t,i}$, $\Delta V_{t,i}=0$, $\Delta b_{t,i}=0, \Delta t_{t,i}=0$ 
        \ENDIF
    \ENDFOR
  \end{algorithmic}
\end{algorithm}

\subsection{Theoretical Results} \label{subsec:proof}
As presented in Section \ref{subsec:alg}, our algorithm first utilizes a homogeneity test to cluster similar clients in a heterogeneous environment. We prove that with our homogeneity test, Algorithm \ref{alg:clustering} correctly identifies the underlying clusters.
\begin{theorem}[\textbf{Clustering Correctness}] \label{thm:clustering_correctness}
Under the condition that we set the homogeneity test threshold $\upsilon^c\geq F^{-1}(1-\frac{\delta}{N^2}, df, \psi^c)$, with probability at least $1-\delta$, we have  
$\hat{\cC} = \cC$.
\end{theorem} 
$F^{-1}(\cdot)$ is the inverse of the CDF of the non-central $\chi^2$ distribution, and $\psi^c \doteq \frac{1}{\sigma^2}$. We provide the complete proof of Theorem \ref{thm:clustering_correctness} in Appendix \ref{appendix:proof:clustering_correctness}. Moreover, Algorithm \ref{alg:simplified} adopts a UCB-based arm selection, which requires the construction of a confidence ellipsoid.
\begin{lemma}[\textbf{Confidence Ellipsoids}] \label{confidence_sets}
Suppose client $i$ is a member of cluster $\hat{C}_k\in \hat{\cC}$, and is therefore collaborating with clients $j\in \hat{C_k}$. For any $\delta>0$, with probability at least $1-\delta$, for all $t\geq 0$ and all clients $i\in {N}$, $\theta_i^*$ lies in the set:
    \begin{equation*}
    \begin{split}
    \beta_{t,i} = \biggl\{ \theta \in \bR^d : \big\|\hat{\theta}_{t,i}- \theta\big\|_{\overline{V}_{t,i}} \leq \sigma\sqrt{2 \log \biggr( \frac{\det(\overline{V}_{t,i})^{1/2}}{ \det(\lambda I)^{1/2}\delta}\biggr)} \\
    + \sqrt{\lambda} + \bigg\|\sum_{j \in \hat{C}_{k} \setminus  \{i\}}^{} \mathbf{X}_{j}^\top\mathbf{X}_j(\theta_{j}^{*}-\theta_{i}^{*})\bigg\|_{\overline{V}_{t,i}^{-1}}\biggr\}
    \end{split}
    \end{equation*}
\end{lemma}

We provide the complete proof of Lemma \ref{confidence_sets} in Appendix \ref{appendix:proof:confidence_sets}. 

Our algorithm enables collaboration among heterogeneous clients, which introduces extra biases represented by the term $H = \|\sum_{j \in \hat{C}_{k}\setminus \{i\}} \mathbf{X}_{j}^\top\mathbf{X}_j(\theta_{j}^{}-\theta_{i}^{})\|_{\overline{V}_{t,i}^{-1}}$.  With improved analysis compared with \cite{lidyclu}, we utilize our cluster estimation procedure Algorithm \ref{alg:clustering} to control the magnitude of the bias term $H$ by judiciously picking the threshold $\upsilon^{c}$.

Then, based on the constructed confidence ellipsoid, we prove Theorem \ref{thm:regret_comm}, which provides upper bounds of the cumulative regret $R_T$ and communication cost $C_T$ incurred by \model. While the good/bad epoch decomposition used in our analysis is first introduced by \citet{dislinucb}, additional care needs to be taken when bounding the extra regret introduced by the delays in serving clusters before they can be removed from the queue. We present the complete proof of Theorem \ref{thm:regret_comm} in Appendix \ref{appendix:proof:regret_comm}.
\begin{theorem}[\textbf{Regret and Communication Cost}]\label{thm:regret_comm}
With an exploration phase length of $T_0 = \frac{16\psi^d \sigma^2}{\lambda_c \gamma^2}$, with probability $1-\delta$ our protocol achieves a cumulative regret of 
\begin{equation}
\begin{split}
    R_T = O\biggr( \frac{N\psi^d \sigma^2}{\lambda_c \gamma^2} + \sum_{k=1}^M d\sqrt{|C_k|T}\log^{2}(|C_k|T) \\ + d|C_k|^2M\log(|C_k|\,T)\biggr),
\end{split}
\end{equation}
where $\psi^{d}=F^{-1}\big(\frac{\delta}{N^{2}(M-1)};d,\upsilon^{c}\big)$,
with communication cost
\begin{equation} 
    C_{T}= O(Nd^2) + \sum_{k=1}^M O(|C_k|^{1.5} \cdot d^3).
\end{equation}

\end{theorem}

\begin{remark}
The regret upper-bound has three components. The first term is our version of the "problem hardness" \citep{lidyclu, gentile2014online} which is independent of T. This "hardness" factor is determined by the cluster separation parameter $\gamma$  from Assumption \ref{ass:separation}.
The second term is the standard regret upper bound from centralized clustered bandit algorithms \citep{lidyclu, gentile2014online}. The third term arises from the potential waiting time clusters may experience in the queue before the server serves them. Our communication cost matches that of an idealized algorithm executing DisLinUCB\citep{dislinucb} within each ground-truth cluster. 
\end{remark}
We compare our regret and communication upper-bound under three cases. \textbf{Case 1 - Single cluster:} Setting $M=1$ reduces the problem to a nearly homogeneous setting, where every client is within $\epsilon$ of everyone else.
Under this setting, our regret becomes $\Tilde{O}(d\sqrt{NT})$ where logarithmic factors and factors that do not depend on $T$ (since it is assumed that $T \gg N$) are omitted in $\Tilde{O}$. Additionally communication cost becomes $O(N^{1.5} d^3)$.
Our algorithm matches the regret and communication cost of \citep{dislinucb}, which is designed for homogeneous clients. 
\textbf{Case 2 - N clusters:} Setting $M=N$ reduces the problem to a completely heterogeneous setting, where no client can benefit from collaboration as each client is at least $\epsilon$ away from others. 
Algorithm \ref{alg:simplified} has regret $\Tilde{O}(dN\sqrt{T})$, which recovers the regret of running LinUCB \citep{abbasi2011improved} independently on each client. In this setting our communication becomes $O(Nd^3)$. 
\textbf{Case 3 - Equal Size Clusters:} Setting $|C_k| = N / M, \forall k$ gives us $M$ clusters of equal size. In this setting our regret becomes $\Tilde{O}(d\sqrt{MNT})$, where the first term recovers the results presented in \citet{gentile2014online, lidyclu}. Our communication becomes $O(d^3N^{1.5} / \sqrt{M})$.


\subsection{Empirical Enhancements: \modelenhance} \label{method:enhance}
In this section, we describe the details of our proposed empirical enhancements to our \model{} algorithm, where we perform re-clustering to improve the quality of estimated clusters of clients and replace the first-in-first-out queue with a priority queue to help clusters where a shared model update can most rapidly reduce regret for clients in that cluster. We present the details of our enhanced algorithm \modelenhance\ in Algorithm \ref{alg:data-depsimplified}.

\subsubsection{Data-Dependent Clustering}
We propose a data-dependent clustering procedure to enhance collaboration among clients with similar observational histories. Our homogeneity test for cluster formation ensures an upper bound on the bias term $H$, as outlined in Lemma \ref{confidence_sets}. This term depends on the differences in underlying parameters ($\theta_i^*$ vs., $\theta_j^*$) and each client's observation history. For instance, if client $j$'s observations are in the null space of $(\theta_{j}^{*}-\theta_{i}^{*})$, collaborating with client $j$ will not introduce excessive bias to client $i$. But without further assumptions about the context vector sequence, we must conservatively assume in our original design that every client $j$'s entire observation history aligns with $(\theta_{j}^{*}-\theta_{i}^{*})$.

Previously, we used a homogeneity test with threshold $\epsilon = \frac{1}{N\sqrt{T}}$ to verify clients' collaboration across all timesteps. Now, we can relax the homogeneity test threshold to check if two clients can collaborate at a specific timestep $t$ by examining if $\|\theta_i - \theta_j\|\leq \epsilon = {1} / \left({N\sqrt{\lambda_{\max}(\mathbf{X}_{j}^\top\mathbf{X}_j)}}\right)$.
To achieve this, we modify our algorithm to forgo single round cluster estimation. Instead, every time a client requests collaboration, we re-cluster the clients using the data-dependent thresholds for our pairwise homogeneity tests. By making these thresholds data-dependent, each client can collaborate with more neighbors earlier, boosting overall collaborative benefits in our learning system.

\begin{algorithm}[!h]
    \caption{Data-Dependent Clustering}\label{alg:data-depclustering}
  \begin{algorithmic}[1]
        \STATE Re-initialize client graph $\mathcal{G}$ with no edges
        \FOR{$(i,j) \in N$}
            \STATE Server Computes $\upsilon^c = F^{-1}(1-\frac{\delta}{N^2}, df, \psi^c)$ where 
            \STATE $\psi^c = \frac{\epsilon^2}{\sigma^2} \lambda_\text{max} (\mathbf{X}_2^\top \mathbf{X}_2(\mathbf{X}_1^\top \mathbf{X}_1 + \mathbf{X}_2^\top \mathbf{X}_2)^{-1} \mathbf{X}_1^\top \mathbf{X}_1)$
            \IF{$s(\cH_{t,i},\cH_{t,j}) \leq \upsilon^{c}$}
                \STATE Add edge $e(i,j)$ to $\mathcal{G}$
            \ENDIF
        \ENDFOR
        \STATE $\hat{\cC} = \{\hat{C}_1, \hat{C}_2, ... \hat{C}_{\hat{M}}\}=$ maximal\_cliques($\mathcal{G}$)
        \STATE Set $\mathcal{K}_{i}$ = $\{k: i\in \hat{C_k}\}$ for each client $i$
        \STATE Cluster communication thresholds $\mathcal{D}=[D_1,...,D_{\hat{M}}]$ where $D_k = (T \log |\hat{C}_k|T) / (d |\hat{C}_k|)$
  \end{algorithmic}
\end{algorithm}
\begin{algorithm}[!h]
    \caption{\modelenhance}\label{alg:data-depsimplified}
  \begin{algorithmic}[1]
    \STATE \textbf{Input:} $T$, $\delta \in (0,1)$, regularization parameter $\lambda>0$
    \STATE \textbf{Initialize Clients:} $\forall i \in N$: $V_{0,i}=\textbf{0}_{d \times d}, b_{0,i}=\textbf{0}_{d}, \cH_{0,i}=\emptyset,\Delta V_{0,i}=\textbf{0}_{d \times d}, \Delta b_{0,i}=\textbf{0}_{d},\Delta t_{i,0}=0$, $\mathcal{K}_{i} = \emptyset$
    \STATE \textbf{Initialize Server: }Client graph $\mathcal{G}$ with $N$ nodes;
    \STATE Initialize empty Priority Queue $Q$; 
    \FOR{$t=T_0+1,...,T$}
        \FOR{Client $i \in N$}
            \STATE $\overline{V}_{t-1,i} = V_{t-1,i} + \lambda I$, $\hat{\theta}_{t-1,i} = \overline{V}_{t-1,i}^{-1}\, b_{t-1,i}$
            \STATE Choose arm $x_{t,i} \in \cA_{t,i}$ by Equation \ref{eq:UCB} observe reward $y_{t,i}$ 
            \STATE Update agent $i$: $\cH_{t,i}=\cH_{t-1,i}\cup (x_{t,i},y_{t,i})$, $V_{t,i}\pluseq x_{t,i}x_{t,i}^{\top}$, $b_{t,i}\pluseq x_{t,i}y_{t,i}$, \STATE $\Delta V_{t,i}\pluseq x_{t,i}x_{t,i}^{\top}$, $\Delta b_{t,i}\pluseq x_{t,i}y_{t,i}, \Delta t_{t,i}\pluseq 1$ 
            \IF{$\Delta t_{t,i} \log(\det(V_{t,i})/\det(V_{t,i}-\Delta V_{t,i}))\geq D_k$}
                \STATE Empty Priority Queue $Q$; 
                \STATE Every client $i\in N$ sends  $\Delta V_{t,j}$ and $\Delta b_{t,j}$ to server
                \STATE Data Dependent Cluster Estimation (Algorithm \ref{alg:data-depclustering})
                \STATE Send collaboration request to server, which then adds $\hat{C}_k \forall k \in \mathcal{K}_i$ to $Q$
            \ENDIF
        \ENDFOR
        \IF{Q is non-empty}
            \STATE Server pops cluster $\hat{C_k}= \argmax_{\hat{C}_k \in \hat{\cC}} \sum_{i\in \hat{C}_k}\Delta t_{t,i} \log(\frac{\det(V_{t,i})}{\det(V_{t,i}-\Delta V_{t,i})})$ from Q
            \STATE Server Computes: $V_{t,sync} = \sum_{j \in \hat{C_k}}\Delta V_{t,j}$, $b_{t,sync} = \sum_{j \in \hat{C_k}}\Delta b_{t,j}$
            \STATE Each client in $\hat{C}_k$ receives $V_{t,sync}$ and $b_{t,sync}$ from the server and updates their local model
            \STATE $V_{t,i}\pluseq V_{t,sync} - \Delta V_{t,i}$, $b_{t,i}\pluseq b_{t,sync} - \Delta b_{t,i}$, $\Delta V_{t,i}=0$, $\Delta b_{t,i}=0, \Delta t_{t,i}=0$ 
        \ENDIF
    \ENDFOR
  \end{algorithmic}
\end{algorithm}

\subsubsection{Priority Queue}
The second enhancement to \model{} involves utilizing a priority queue instead of a FIFO queue to determine the order in which to serve clusters requesting collaboration. As is demonstrated in \citep{dislinucb, Liasynch}, the cumulative regret incurred by a federated bandit algorithm is determined by the determinant ratios of the clients within the system: $\Delta t_{t,i} \log\big(\frac{\det(V_{t,i})}{\det(V_{t,i}-\Delta V_{t,i}}\big)$. Since the central server cannot assist all clusters at once, determinant ratios of awaiting clients can increase as they linger in the queue. In the original \model{}, clusters are attended based on their request order. However, an earlier-joining cluster might have a slower regret accumulation compared to a later one with a larger and faster growing determinant ratio.
By utilizing a priority queue that serves the clusters based on: $\argmax_{\hat{C}_k \in \hat{\cC}} \sum_{i\in \hat{C}_k}\Delta t_{t,i} \log(\frac{\det(V_{t,i})}{\det(V_{t,i}-\Delta V_{t,i}})$
the server ensures clusters are addressed in an order that minimizes the system-wide cumulative regret.

\section{Experiments} \label{sec:experiments}
In this section, we investigate the empirical performance of \model \ and \modelenhance{}, by comparing them against several baseline models on both simulated and real-world datasets.

\subsection{Baselines} In our evaluation, we compare our proposed \model\ algorithm with several representative algorithms from both the clustered and federated bandit learning domains. We compare against, LinUCB algorithm from \citep{abbasi2011improved}, DisLinUCB \citep{dislinucb}, FCLUB\_DC \citep{liufederatedonlineclusteringofbandits}, and DyClu \citep{lidyclu}. To ensure compatibility with our setting, we set the number of local servers in FCLUB\_DC to be equal to the number of clients.

\subsection{Synthetic Dataset} \label{experiment:synthetic}
We first present the results of our empirical analysis of \model{} and \modelenhance{} on a synthetic dataset.

\paragraph{Synthetic Dataset Generation} \label{experiment:synthetic-generation}
In this section, we describe the pre-processing procedure for the synthetic dataset used in Section \ref{experiment:synthetic}. We first create an action pool $\{x_k\}_{k=1}^{K}$ where $x$ is sampled from $N(0_d, I_d)$. To create a set of $N$ clients in accordance with our environment assumptions, we first  sample $M$ cluster centers $\{\theta_m\}_{m=1}^{M}$ from $N(0_d, I_d)$ that are $\gamma + 2\epsilon$ away from each other (enforced via rejection sampling). Then, we randomly assign each client index $i\in N$ to one of the $M$ clusters. To generate each $\theta_i$, we first sample a vector on the unit d-sphere, then we scale it by a value uniformly sampled from $[0,\epsilon]$ with $\epsilon = 1 / N\sqrt{T}$ and add it to the cluster center $\theta_k$ corresponding to the cluster client $i$ was assigned to. At each time step $t = 1,2,...T$ for each client in $[N]$ is presented a subset of $25$ arms are sampled from $\{x_k\}_{k=1}^{K}$, and shared with the client. The reward of the selected arm is generated by the linear function governed by the corresponding bandit parameter and context. In our experiments, we chose $d=25$, $K=1000$, $N=50$, $M=5$,  and $T=3000$. Since we conducted our experiment in a synthetic environment, we utilized the known values of $\gamma = 0.85$, $\sigma = 0.1$, and chose $\lambda= 0.1$, $\delta = 0.1$ for our algorithm's hyper-parameters. Note that while we utilized the known values for $\sigma$ and $\gamma$ in our experiment, these hyper-parameters can be tuned in practice using the ``doubling-trick", where the algorithm is repeatedly run in the same environment with increasing horizons \citep{aueradversarial1995,besson2018doubling}.  We present additional sensitivity analysis of the environmental parameters in Appendix \ref{appendix:sensitivity} due to the space limit.

\paragraph{Results}

In Figure \ref{fig:sim_results_reg}, we compare the accumulated regret of the different bandit algorithms on the simulated dataset. \model\ and \modelenhance{} outperform the other decentralized bandit baselines, with \modelenhance{} achieving a regret that is closest to the state-of-the-art centralized clustering bandit algorithm, DyClu. We observe that DisLinUCB experiences linear regret in our heterogeneous environment. While N-Independent LinUCB achieves sublinear regret, its cumulative regret is higher than our \model\ due to the absence of collaboration among similar clients. \model\ outperforms FCLUB\_DC, underscoring the strength of our federated clustered bandit approach in a heterogeneous setting. FCLUB\_DC's fixed clustering schedule results in delayed cluster identification, diminishing the quality of client collaboration. Notably, while FCLUB\_DC presumes the central server aids all local client clusters concurrently at each step, \model\ still excels despite adhering to a single model assumption in federated learning. We further present an ablation study on our empirical enhancements in Appendix \ref{appendix:empirical}.
\begin{figure}[h]
     \centering
     \begin{subfigure}[b]{0.49\textwidth}
         \centering         
         \includegraphics[width=0.99\textwidth,height=5.2cm]{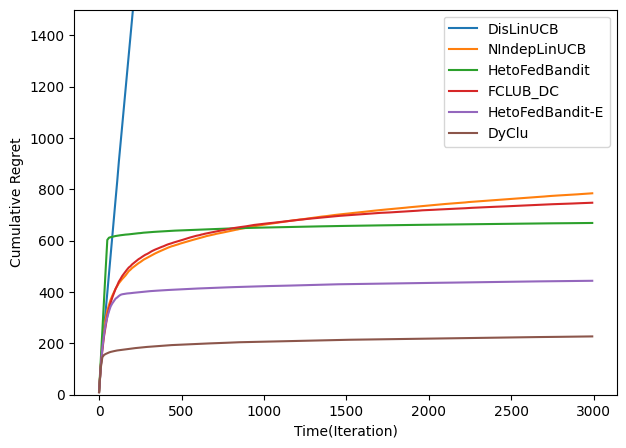}
         \caption{Accumulative Regret}
         \label{fig:sim_results_reg}
     \end{subfigure}
     \hfill
     \begin{subfigure}[b]{0.49\textwidth}
         \centering
         \includegraphics[width=0.99\textwidth, height=5.2cm]{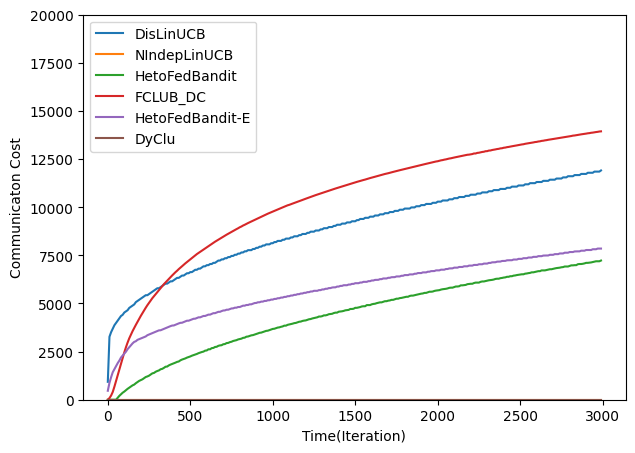}
         \caption{Communication Cost}
         \label{fig:sim_comm_cost}
     \end{subfigure}
     \hfill
       \caption{Experimental Results on Simulated Dataset}
        \label{fig:results_sim}
        \vspace{-4mm}
\end{figure}


In Figure \ref{fig:sim_comm_cost}, we observe that our algorithms exhibit the lowest communication cost among baselines while achieving encouraging regret. This demonstrates the communication efficiency of our approach, which is a critical factor in distributed systems. Notably, \modelenhance{} has a higher communication cost compared to \model{}, because the dynamic re-clustering requires the additional sharing of sufficient statistics from clients outside the cluster that requests collaboration. 

\subsection{LastFM Dataset} \label{experiment:realworld}
In this section, we present the results of our empirical analysis of \model{} and \modelenhance{} on the LastFM dataset, demonstrating their effectiveness in distributed recommender systems. 
\paragraph{LastFM Dataset} The dataset used in this experiment is extracted from the LastFM-2k dataset, which originally contains 1892 clients (users) and 16632 items (artists) \citep{Cantador:RecSys2011}. Each ``client" can be considered as an edge device serving a particular user in a distributed recommender system. The ``listened artists'' of each client are treated as positive feedback. To adapt this dataset for our experiments, we kept clients with over 350 observations, resulting in a dataset with $N=75$ clients and $T=41284$ interactions. The dataset was pre-processed following the procedure in \citep{cesa2013gang} to accommodate the linear bandit setting (with $d=25$ and the action set $K=25$). Since the environmental parameters $\sigma, \gamma$ are unknown for this real-world dataset, we directly tuned the values of our test threshold $\upsilon^c = 0.01$, $T_0 = 5000$, and $\alpha_{t,i} = \alpha = 0.3 \ \forall i \in \left[N\right]$, $\forall t \in T$ using a grid search.
\paragraph{Results} 
 We show that our models group users with similar musical preferences for collaborative model learning, enhancing recommendation quality compared to other distributed bandit learning methods. In Figure \ref{fig:results_real}, we present the normalized cumulative rewards and communication costs of the federated bandit algorithms on the LastFM dataset. We observe that \modelenhance\ outperforms the other decentralized bandit baselines on the real-world dataset, achieving the highest average normalized reward. In line with observations made using synthetic datasets, DisLinUCB's performance is suboptimal in environments with heterogeneous clients, demonstrated by NIndepLinUCB outperforming DisLinUCB. Moreover, while the normalized cumulative reward of FCLUB\_DC shows an improving trend over time, its prefixed-clustering schedule delays the identification of the underlying cluster structure compared to \model{} and \modelenhance{}. 

Notably, our basic algorithm design, \model{}, falls short of NIndepLinUCB on real-world data due to its single-timestep cluster estimation. In a simulated environment, where context vectors adhere closely to Assumption \ref{as:contextreg}, one-time cluster estimation post-exploration is usually sufficient. However, in real-world datasets, the distribution of context vectors may evolve over time, leading to potential inaccuracies in the clusters initially estimated after the exploration phase. This underlines the importance of our empirical enhancements, which incorporates dynamic data-dependent re-clustering, demonstrating its ability to adapt to shifts in the observed context distribution.

In Figure \ref{fig:real_comm_cost}, we observe that our algorithms once again exhibit the lowest communication cost among the compared baselines. Similar to our observations in Section \ref{experiment:synthetic}, \modelenhance{} has a higher communication cost compared to \model{}.


\begin{figure}[h]
     \centering
     \begin{subfigure}[b]{0.49\textwidth}
         \centering
         \includegraphics[width=\textwidth,height=5.2cm]{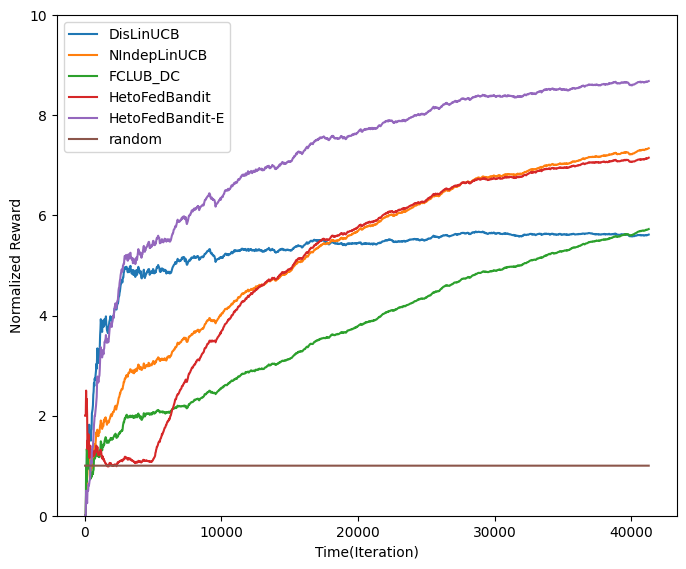}
         \caption{Normalized Accumulated Reward}
         \label{fig:real_results_reg}
     \end{subfigure}
     \hfill
     \begin{subfigure}[b]{0.49\textwidth}
         \centering
         \includegraphics[width=\textwidth, height=5.2cm]{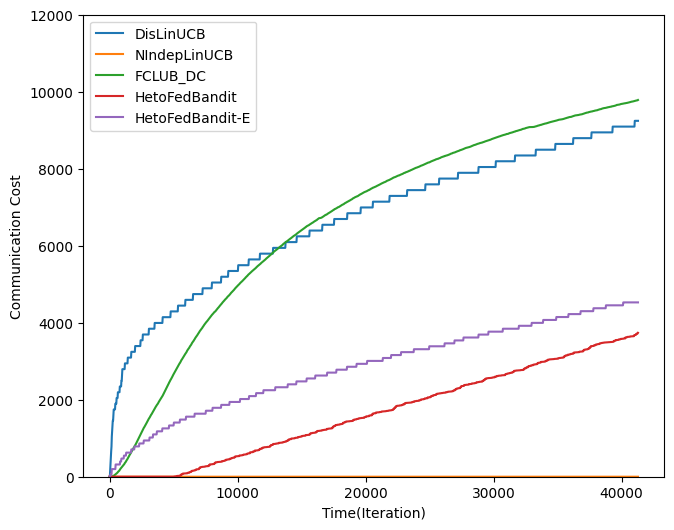}
         \caption{Communication Cost}
         \label{fig:real_comm_cost}
     \end{subfigure}
     \hfill
        \caption{Experimental Results on LastFM Dataset}
        \label{fig:results_real}
\end{figure}

\section{Conclusion}

In this work, we address the challenge of heterogeneous clients in federated bandit learning by introducing \model.\ Our approach combines the strengths of federated learning and collaborative bandit learning, enabling efficient communication and learning among clients with diverse objectives. We demonstrate through rigorous theoretical analysis that participating clients achieve regret reduction compared to their independent learning across various environmental settings, thereby motivating all clients to participate in such a federated learning system. We also empirically demonstrate that our algorithm achieves encouraging performance compared to existing federated bandit learning solutions on both simulated and real-world datasets. 

Our work not only addresses the limitations of existing federated bandit learning solutions, but also opens up new possibilities for practical applications in distributed systems. Our current approach requires that each client within the federated learning system trusts the central server to only facilitate collaboration for social good. However, in the real world, this ``trust'' is not something that should be naively assumed. To realize a truly federated model of bandit learning, the power to decide on collaboration should be transferred to the clients. In this regard, future research should consider viewing federated learning through the lens of mechanism design, so that each client perceives participating in the federated learning system as their best course of action.  


\acks{We thank the anonymous reviewers for their insightful suggestions and comments. This material is based upon work supported by the NSF Graduate Research Fellowship under Grant No. 1842490 and NSF Award IIS-2213700 and IIS-2128019.}


\newpage

\appendix
\appendix
\onecolumn

\section{Technical Lemmas}
In this section, we introduce the technical lemmas utilized in the subsequent proofs in this paper.
\label{sec:appendix}
\begin{lemma}[Lemma 11 in \cite{abbasi2011improved}] 
Let $\{X_t\}_{t=1}^{\infty}$ be a sequence in $\bR^d$, $V$ is  a $d \times d$ positive definite matrix and define $\bar{V}_t = V + \sum_{s=1}^{t}X_sX_s^\top$, where $V=\lambda I$. Additionally we have that $\lambda_{\min}(V)\geq \max(1,L^2)$ and $\|X_t\|_2 \leq L$ for all $t$, then
\begin{equation}
    \log\biggr(\frac{\det(\bar{V_n})}{\det(V)}\biggr)\leq \sum_{t=1}^T \|X_t\|^2_{\bar{V}_{t-1}^{-1}} \leq 2\log\biggr(\frac{\det(\bar{V_n})}{\det(V)}\biggr).
    \end{equation}
\end{lemma}

\begin{lemma}[Theorem 1 of \cite{abbasi2011improved}] \label{lem:self_normalized_bound}
Let $\{\cF_{t}\}_{t=0}^{\infty}$ be a filtration. Let $\{\eta_{t}\}_{t=1}^{\infty}$ be a real-valued stochastic process such that $\eta_{t}$ is $\cF_{t}$-measurable, and $\eta_{t}$ follows conditionally zero mean $R$-sub-Gaussian for some $R \geq 0$.
Let $\{X_{t}\}_{t=1}^{\infty}$ be an $\bR^{d}$-valued stochastic process such that $X_{t}$ is $\cF_{t-1}$-measurable. Assume that $V$ is a $d \times d$ positive definite matrix. For any $t > 0$, define
\begin{align*}
    V_{t}=V+\sum_{\tau=1}^{t} X_{\tau} X_{\tau}^{\top} \quad \cS_{t}=\sum_{\tau=1}^{t}\eta_{\tau} X_{\tau}.
\end{align*}
Then for any $\delta >0$, with probability at least $1-\delta$,
\begin{align*}
    ||\cS_{t}||_{V_{t}^{-1}} \leq  R \sqrt{2\log{\frac{\det(V_{t})^{1/2}}{\det(V)^{1/2}\delta}}}, \quad \forall t\geq 0.
\end{align*}
\end{lemma}

\begin{lemma}[Determinant-Trace Inequality] \label{lem:det_trace}
Suppose $X_1, X_2, ..., X_t \in \bR^d$ and for any $1\leq \tau \leq t$, $\|X_\tau\|_2 \leq L$, Let $\overline{V_t} = \lambda I + \sum_{\tau=1}^{t} X_{\tau} X_{\tau}^{\top}$ for some $\lambda > 0$. Then,
\begin{align*}
    \det(\overline{V_t}) \leq (\lambda + tL^2/d)^d.
\end{align*}
\end{lemma}

\begin{lemma}[Lemma 12 from \citep{lidyclu}] \label{lem:type2Cluster}
When the underlying bandit parameters $\theta_i^*$ and $\theta_j^*$ of two observation sequence $\cH_{t-1,i}$ and $\cH_{t-1,j}$ from client $i$ and $j$ are not the same, the probability that the cluster identification phase clusters them together corresponds to the type-II error probability given in Lemma \ref{lem:type2Similarity}, which can be upper bounded by:
\begin{equation*}
P\left(S(\cH_{t-1,i},\cH_{t-1,j}) \leq \upsilon^{c}\big|\|\theta_i^*-\theta_j^*\|>\epsilon\right) \leq F(\upsilon^{c};d,\psi^{d}),
\end{equation*}
under the condition that both $\lambda_{\min}(\sum_{(\bx_{k},y_{k}) \in \cH_{t-1,i}}\bx_{k}\bx_{k}^{\top})$ and $\lambda_{\min}(\sum_{(\bx_{k},y_{k}) \in \cH_{t-1, 2}}\bx_{k}\bx_{k}^{\top})$ are at least $\frac{2\psi^{d}\sigma^{2}}{\gamma^{2}}$.
\end{lemma}

\begin{lemma}[Lemma B1 from \cite{Liasynch}] \label{lem:min_eig}
Denote the number of observations that have been used to update $\{V_{i,t}, b_{i,t}\}$ as $\tau_{i}$, i.e., $V_{i,t}=\lambda I + \sum_{s=1}^{\tau_{i}}\bx_{s} \bx_{s}^{\top}$.
Then under Assumption \ref{as:contextreg}, with probability at least $1-\delta$, we have:
\begin{align*}
    \lambda_{\min}(V_{i,t}) \geq \lambda + \frac{\lambda_{c} \tau_{i}}{8},
\end{align*}
$\forall \tau_{i} \in\{\tau_{min},\tau_{min}+1,\dots,T\}, i \in [N]$, where $\tau_{min}=\lceil \frac{64}{3 \lambda_{c}}\log(\frac{2NTd}{\delta}) \rceil$.
\end{lemma}
\begin{lemma}{Lemma 3 from \cite{lidyclu}}\label{lem:type2Similarity}
When 
$\bX_{1}$ and $\bX_{2}$ are rank-sufficient,
the type-II error probability can be upper bounded by,
\small
\begin{equation*}
    P\big(s(\cH_{t-1,1},\cH_{t-1,2}) \leq \upsilon \,| \,\|\theta_{1}-\theta_{2}\| > \epsilon\big) \leq 
             F\Big(\upsilon^c;d,\frac{||\theta_{1}^*-\theta_{2}^*||^{2}/\sigma^{2}}{{1}/{\lambda_{min}(\bX_{1}^{\top} \bX_{1})}+{1}/{\lambda_{min}(\bX_{2}^{\top} \bX_{2})}}\Big).\\
\end{equation*}
\normalsize
\end{lemma}

\section{Proof of Theorem \ref{thm:clustering_correctness}} \label{appendix:proof:clustering_correctness}
In this section, we provide the full proof of Theorem \ref{thm:clustering_correctness}, which states that utilizing our homogeneity test with threshold $\upsilon^c\geq F^{-1}(1-\frac{\delta}{N^2}; df, \psi^c)$, after the exploration phase of length $T_0= \frac{16\psi^d \sigma^2}{\lambda_c \gamma^2}$, the clusters $\hat{\cC} = \{\hat{C}_1,\hat{C}_2,...,\hat{C}_{\hat M} \}$ estimated by \model\  match the ground-truth clusters of the environment $\cC = \{C_1,C_2,...,C_M \}$. 

The homogeneity test statistic $s(\cH_{t-1,1},\cH_{t-1,2})$ follows a non-central $\chi^{2}$ distribution $s(\cH_{t-1,1},\cH_{t-1,2}) \sim \chi^{2}(df, \psi)$, where the degree of freedom 
\begin{equation*}
df=rank(\bX_{1})+rank(\bX_{2})-rank\left(\left[\begin{matrix}\bX_{1} \bX_{2}\end{matrix}\right]\right),
\end{equation*}
and the non-centrality parameter 
\begin{equation*}
    \psi=\frac{1}{\sigma^{2}}\left[\begin{matrix}\bX_{1}\theta_{1} \bX_{2}\theta_{2}^*\end{matrix}\right]^{\top}\left[\bI_{t_{1}+t_{2}}-\left[\begin{matrix}\bX_{1} \bX_{2}\end{matrix}\right]\left(\bX_{1}^{\top} \bX_{1}+\bX_{2}^{\top} \bX_{2}\right)^{-}\left[\begin{matrix}\bX_{1}^{\top} &\bX_{2}^{\top} \end{matrix}\right]\right]\left[\begin{matrix}\bX_{1}\theta_{1} \bX_{2}\theta_{2}^*\end{matrix}\right]
\end{equation*}
\citep{lidyclu}.


Based on the definition and properties of the test statistic, we next prove two corollaries. First we will prove that with high probability $ \cC \subseteq\hat{\cC}$. Then we will prove that $\hat{\cC} \subseteq \cC$. As a result, the conjunction of these events holding simultaneously demonstrates that $\hat{\cC} = \cC$, proving that our estimated clusters are correct with a high probability.

\subsection{Lower Bounding $P(\cC \subseteq \hat{\cC})$}
Recall that based on our cluster definition presented in Assumption \ref{ass:proximity}, all clients that belong to the same cluster are within $\epsilon$ of each other. We denote the ground-truth client graph $G^*$ as the graph where $\exists e(i,j) \in G^* \, ~ \forall i,j \in N$ where $\|\theta_{i}^* - \theta_{j}^*\| \leq \epsilon$. By Assumption \ref{ass:separation}, we know that clients that do not belong to the same cluster are separated by $\gamma$, so that the ground-truth clusters $\cC$ are the maximal cliques of $G^*$. Thus, in order to prove that $P(\cC \subseteq \hat{\cC})$, we need to show that the set of edges in the ground-truth client graph $G^*$ is a subset of the edges in the estimated client graph $G$. To achieve this, we need to prove an upper-bound of the type-I error probability of the homogeneity test, which corresponds to the probability that our algorithm fails to cluster two clients together when the underlying bandit parameters $\|\theta_{i}^* - \theta_{j}^*\| \leq \epsilon$.

\begin{lemma}\label{lem:type1Similarity}
The type-I error probability of the test can be upper bounded by:
\begin{equation*}
P\big(s(\cH_{t-1,1},\cH_{t-1,2}) > \upsilon\,|\, \|\theta_{1}^*-\theta_{2}^*\| \leq \epsilon\big) \leq 1-F(\upsilon;df,\psi^c),
\end{equation*}
where $F(\upsilon;df,\psi^c)$ denotes the cumulative density function (CDF) of distribution $\chi^{2}(df,\psi^c)$ evaluated at $\upsilon$, and $\psi^c:=\frac{1}{\sigma^2}$ denotes its non-centrality parameter.
\end{lemma}

\begin{proof}
Denote $\zeta=\theta_{2}^*-\theta_{1}^*$. Then $\theta_{2}^*=\theta_{1}^*+\zeta$. When $\|\zeta\| \leq \epsilon$, the non-centrality parameter $\psi$ becomes:
\begin{equation*}
\begin{split}
\psi = \frac{1}{\sigma^{2}}& \left[\begin{matrix}\bX_{1}\theta_{1}^* \\ \bX_{2}(\theta_{1}^*+\zeta)\end{matrix}\right]^\top\left[\bI_{t_{1}+t_{2}}-\left[\begin{matrix}\bX_{1} \\ \bX_{2}\end{matrix}\right]\left(\bX_{1}^\top\bX_{1}+\bX_{2}^\top\bX_{2}\right)^{-1}\left[\begin{matrix}\bX_{1}^\top & \bX_{2}^\top\end{matrix}\right]\right]\left[\begin{matrix}\bX_{1}\theta_{1}^* \\ \bX_{2}(\theta_{1}^*+\zeta)\end{matrix}\right], \\
\sigma^{2}\psi= & \left[\begin{matrix}\bX_{1}\theta_{1}^* \\ \bX_{2}\theta_{1}^*\end{matrix}\right]^\top\left[\bI_{t_{1}+t_{2}}-\left[\begin{matrix}\bX_{1} \\ \bX_{2}\end{matrix}\right]\left(\left[\begin{matrix}\bX_{1}^\top & \bX_{2}^\top\end{matrix}\right]\left[\begin{matrix}\bX_{1} \\ \bX_{2}\end{matrix}\right]\right)^{-1}\left[\begin{matrix}\bX_{1}^\top & \bX_{2}^\top\end{matrix}\right]\right]\left[\begin{matrix}\bX_{1}\theta_{1}^* \\ \bX_{2}\theta_{1}^*\end{matrix}\right] \\
& + \left[\begin{matrix}\bX_{1}\theta_{1}^* \\ \bX_{2}\theta_{1}^*\end{matrix}\right]^\top\left[\bI_{t_{1}+t_{2}}-\left[\begin{matrix}\bX_{1} \\ \bX_{2}\end{matrix}\right]\left(\left[\begin{matrix}\bX_{1}^\top & \bX_{2}^\top\end{matrix}\right]\left[\begin{matrix}\bX_{1} \\ \bX_{2}\end{matrix}\right]\right)^{-1}\left[\begin{matrix}\bX_{1}^\top & \bX_{2}^\top\end{matrix}\right]\right]\left[\begin{matrix}0 \\ \bX_{2}\zeta\end{matrix}\right] \\
& + \left[\begin{matrix}0 \\ \bX_{2}\zeta\end{matrix}\right]^\top\left[\bI_{t_{1}+t_{2}}-\left[\begin{matrix}\bX_{1} \\ \bX_{2}\end{matrix}\right]\left(\left[\begin{matrix}\bX_{1}^\top & \bX_{2}^\top\end{matrix}\right]\left[\begin{matrix}\bX_{1} \\ \bX_{2}\end{matrix}\right]\right)^{-1}\left[\begin{matrix}\bX_{1}^\top & \bX_{2}^\top\end{matrix}\right]\right]\left[\begin{matrix}\bX_{1}\theta_{1}^* \\ \bX_{2}\theta_{1}^*\end{matrix}\right] \\
& + \left[\begin{matrix}0 \\ \bX_{2}\zeta\end{matrix}\right]^\top\left[\bI_{t_{1}+t_{2}}-\left[\begin{matrix}\bX_{1} \\ \bX_{2}\end{matrix}\right]\left(\left[\begin{matrix}\bX_{1}^\top & \bX_{2}^\top\end{matrix}\right]\left[\begin{matrix}\bX_{1} \\ \bX_{2}\end{matrix}\right]\right)^{-1}\left[\begin{matrix}\bX_{1}^\top & \bX_{2}^\top\end{matrix}\right]\right]\left[\begin{matrix}0 \\ \bX_{2}\zeta\end{matrix}\right]. \\
\end{split}
\end{equation*}
Since $\left[\begin{matrix}\bX_{1}\theta_{1}^* \\ \bX_{2}\theta_{1}^*\end{matrix}\right]$ is in the column space of $\left[\begin{matrix}\bX_{1} \\ \bX_{2}\end{matrix}\right]$, the first term in the above result is zero. The second and third terms can be shown equal to zero as well using the property that matrix product is distributive with respect to matrix addition, which leaves us only the last term. Therefore, by substituting $\zeta = \theta_{2}^*-\theta_{1}^*$ back, we obtain:
\begin{equation*}
\begin{split}
\psi 
& = \frac{1}{\sigma^{2}}(\theta_{1}^*-\theta_{2}^*)^\top\bX_{2}^\top\bX_{2}(\bX_{1}^\top\bX_{1}+\bX_{2}^\top\bX_{2})^{-1}\bX_{1}^\top\bX_{1}(\theta_{1}^*-\theta_{2}^*),\\
&\leq \frac{1}{\sigma^2} \|\theta_1^*-\theta_2^*\|^2 \lambda_\text{max} (\mathbf{X}_2^\top \mathbf{X}_2(\mathbf{X}_1^\top \mathbf{X}_1 + \mathbf{X}_2^\top \mathbf{X}_2)^{-1} \mathbf{X}_1^\top \mathbf{X}_1) \label{lb_psi}, \\
&\leq \frac{\epsilon^2}{\sigma^2} \lambda_\text{max} (\mathbf{X}_2^\top \mathbf{X}_2(\mathbf{X}_1^\top \mathbf{X}_1 + \mathbf{X}_2^\top \mathbf{X}_2)^{-1} \mathbf{X}_1^\top \mathbf{X}_1). \\
\end{split}
\end{equation*}
The first inequality uses the Rayleigh-Ritz theorem, and the second inequality is a result of Assumption \ref{ass:proximity}. Furthermore, we can use the relation $\bY(\bX+\bY)^{-1}\bX=(\bX^{-1}+\bY^{-1})^{-1}$, where $\bX$ and $\bY$ are both invertible matrices, to further simplify our upper bound for $\psi$. This relation can be derived by taking inverse on both sides of the equation $\bX^{-1}(\bX+\bY)\bY^{-1}=\bX^{-1}\bX\bY^{-1}+\bX^{-1}\bY\bY^{-1}=\bY^{-1}+\bX^{-1}$. This gives us the following:
\begin{align*}
\psi &= \frac{\epsilon^2}{\sigma^2} \lambda_{\max} \left(\big((\mathbf{X}_1^\top \mathbf{X}_1)^{-1}+(\mathbf{X}_2^\top \mathbf{X}_2)^{-1}\big)^{-1}\right),\\
 &\leq \frac{\epsilon^2}{\sigma^2} \frac{1}{\lambda_{\min}((\mathbf{X}_1^\top \mathbf{X}_1)^{-1}+(\mathbf{X}_2^\top \mathbf{X}_2)^{-1})}, \\
 &\leq \frac{\epsilon^2}{\sigma^2} \frac{1}{\lambda_{\min}((\mathbf{X}_1^\top \mathbf{X}_1)^{-1}) + \lambda_{\min}((\mathbf{X}_2^\top \mathbf{X}_2)^{-1})},\\
 &\leq \frac{\epsilon^2}{\sigma^2} \frac{1}{\frac{1}{\lambda_{\max}(\mathbf{X}_1^\top \mathbf{X}_1)} + \frac{1}{\lambda_{\max}(\mathbf{X}_2^\top \mathbf{X}_2)}}, \\
 &= \frac{\epsilon^2}{\sigma^2} \frac{\lambda_{\max}(\mathbf{X}_1^\top \mathbf{X}_1)\times\lambda_{\max}(\mathbf{X}_2^\top \mathbf{X}_2)}{\lambda_{\max}(\mathbf{X}_1^\top \mathbf{X}_1)+\lambda_{\max}(\mathbf{X}_2^\top \mathbf{X}_2)},\\
 &\leq \frac{\epsilon^2}{\sigma^2} \max\{\lambda_{\max}(\mathbf{X}_1^\top \mathbf{X}_1),\lambda_{\max}(\mathbf{X}_2^\top \mathbf{X}_2)\}.
 \end{align*}
 The last inequality holds because
 \begin{align*}
     \frac{\lambda_{\max}(\mathbf{X}_1^\top \mathbf{X}_1)}{\lambda_{\max}(\mathbf{X}_1^\top \mathbf{X}_1)+\lambda_{\max}(\mathbf{X}_2^\top \mathbf{X}_2)}\leq 1 \text{ and } \frac{\lambda_{\max}(\mathbf{X}_2^\top \mathbf{X}_2)}{\lambda_{\max}(\mathbf{X}_1^\top \mathbf{X}_1)+\lambda_{\max}(\mathbf{X}_2^\top \mathbf{X}_2)} \leq 1.
 \end{align*}
 Denote the number of observations in $\mathbf{X}_i$ as $\tau_i$. Furthermore, since $\|x_{t,i}\|\leq 1$, we know that $\lambda_{\max}(\mathbf{X}_i^\top \mathbf{X}_i)\leq \tau_i$. Thus we can further upper bound 
\begin{align*}
 \psi &\leq \frac{\epsilon^2}{\sigma^2} \max\{\lambda_{\max}(\mathbf{X}_1^\top \mathbf{X}_1),\lambda_{\max}(\mathbf{X}_2^\top \mathbf{X}_2)\}, \\
&\leq \frac{\epsilon^2}{\sigma^2} \max\{\tau_{i}, \tau_{j}\}, \\
&\leq \frac{\epsilon^2}{\sigma^2} T.
\end{align*}

Assumption \ref{ass:proximity} tells us that $\epsilon = \frac{1}{N\sqrt{T}}$ for $(i,j)$ in the same cluster ${C}_k$.
\begin{align*}
    \psi &\leq \frac{T}{\sigma^2 N^2T} \leq \frac{1}{\sigma^2} := \psi^c.
\end{align*}
Therefore, when $\|\theta_{1}^*-\theta_{2}^*\|<\epsilon$, the test statistic $s(\cH_{t-1,1},\cH_{t-1,2}) \sim \chi^{2}(df, 0, \psi^c)$. The type-I error probability can be upper bounded by $P\big(s(\cH_{t-1,1},\cH_{t-1,2}) > \upsilon \, \big|\, \|\theta_{1}^*-\theta_{2}^*\|\big) \leq 1-F(\upsilon;df,\psi^c)$, which concludes the proof of Lemma \ref{lem:type1Similarity}.
\end{proof}


\begin{corollary} \label{corr:type1subset}
Under the condition that we set the threshold $\upsilon$ to $\upsilon^c\geq F^{-1}(1-\frac{\delta}{N^2}, df, \psi^c)$, we have 
$P(\cC\subseteq \hat{\cC})\geq 1-\delta$. 

\begin{proof}
In our setting (Assumption \ref{ass:proximity}), all users who are within $\epsilon = \frac{1}{N\sqrt{T}}$ of each other belong to the same ground-truth cluster. Our algorithm uses the pairwise homogeneity test to assess whether each pair of clients is within $\epsilon$ of each other. As we showed in Lemma \ref{lem:type1Similarity}, the type-I error probability of our pairwise neighbor identification is upper-bounded by $1-F(\upsilon;df,\psi^c)$. Therefore, to achieve a type-I error probability of $\delta / N^2$  between two individual clients, we can solve for the required threshold $\upsilon^c$
\begin{align*}
    \frac{\delta}{N^2} &\leq 1-F(\upsilon;df,\psi^c),\\
    \Rightarrow  F(\upsilon;df,\psi^c)\  &\leq 1-\frac{\delta}{N^2}, \\
    \Rightarrow  F^{-1}(1-\frac{\delta}{N^2}, df, \psi^c)&\leq \upsilon^c.
\end{align*}

Taking the union bound over all $N^2$ pairwise tests proves the that the set of edges in the ground-truth client graph $G^*$ is a subset of the edges estimated client graph $G$. Therefore the corollary is proven.
\end{proof}

\end{corollary}

\subsection{Lower Bounding $P(\hat{\cC} \subseteq \cC)$}

In this section, we prove that with high probability $P(\hat{\cC} \subseteq \cC)$. To achieve this, we demonstrate that the set of edges in the estimated client graph $G$ is a subset of the ground-truth edges in $G^*$. To achieve this, we utilize the type-II error probability upper-bound to ensure that with high probability clients with different underlying parameters are not clustered together. Using this type-II error probability, we follow similar steps in Lemma 13 of \citep{lidyclu} to prove: 

\begin{lemma}\label{lem:RealNegativeOutOfPredictedNegative}
If the cluster identification module clusters observation history $\cH_{t-1,i}$ and $\cH_{t-1,j}$ together, the probability that they actually have the same underlying bandit parameters is denoted as $P\big(\|\theta^*_i - \theta^*_j\| \leq \epsilon|s(\cH_{t-1,i},\cH_{t-1,j}) \leq \upsilon^{c}\big)$.
\begin{equation*}
    P\big(\|\theta^*_i - \theta^*_j\| \leq \epsilon |s(\cH_{t-1,i},\cH_{t-1,j}) \leq \upsilon^{c}\big) \geq F(\upsilon^{c};df,\psi^c),
\end{equation*}
under the condition that both $\lambda_{\min}\big(\sum_{(\bx_{k},y_{k}) \in \cH_{t-1,i}}\bx_{k}\bx_{k}^{\top}\big)$ and $\lambda_{\min}\big(\sum_{(\bx_{k},y_{k}) \in \cH_{t-1,j}}\bx_{k}\bx_{k}^{\top}\big)$ are at least $\frac{2\psi^{d}\sigma^{2}}{\gamma^{2}}$, where $\psi^{d}=F^{-1}\big(\frac{(1-F(\upsilon^{c};d,\psi^c))}{M-1};d,\upsilon^{c}\big)$.
\end{lemma}

\begin{proof}

Compared with the type-I and type-II error probabilities given in Lemma \ref{lem:type1Similarity} and \ref{lem:type2Cluster}, the probability $P(\|\theta^*_i - \theta^*_j\| \leq \epsilon|S(\cH_{t-1,i},\cH_{t-1,j}) \leq \upsilon^{c})$ also depends on the population being tested on.

Denote the events $\big\{\|\theta^*_i - \theta^*_j\| > \epsilon\big\} \cap \big\{S(\cH_{t-1,i},\cH_{t-1,j}) > \upsilon^{c}\big\}$ as True Positive ($TP$), $\big\{\|\theta^*_i - \theta^*_j\| \leq \epsilon\big\} \cap \big\{S(\cH_{t-1,i},\cH_{t-1,j}) \leq \upsilon^{c}\big\}$ as True Negative ($TN$), $\big\{\|\theta^*_i - \theta^*_j\| \leq \epsilon\big\} \cap \big\{S(\cH_{t-1,i},\cH_{t-1,j}) > \upsilon^{c}\big\}$ as False Positive ($FP$), and $\big\{\|\theta^*_i - \theta^*_j\| > \epsilon\big\} \cap \big\{S(\cH_{t-1,i},\cH_{t-1,j}) \leq \upsilon^{c}\big\}$ as False Negative ($FN$) of cluster identification, respectively. We can rewrite the probabilities in Lemma \ref{lem:type1Similarity}, \ref{lem:type2Cluster} and \ref{lem:RealNegativeOutOfPredictedNegative} as:
\begin{align*}
P\big(S(\cH_{t-1,i},\cH_{t-1,j}) > \upsilon^{c}|\|\theta^*_i - \theta^*_j\| \leq \epsilon\big) &= \frac{P(FP)}{P(TN+FP)} \leq 1-F(\upsilon^{c};df,\psi^c), \\
P\big(s(\cH_{t-1,i},\cH_{t-1,j}) \leq \upsilon^{c}|\|\theta^*_i - \theta^*_j\| > \epsilon\big) &= \frac{P(FN)}{P(FN+TP)} \leq F(\upsilon^{c};df,\psi^{d}), \\
P\big(\|\theta^*_i - \theta^*_j\| \leq \epsilon|s(\cH_{t-1,i},\cH_{t-1,j}) \leq \upsilon^{c}\big) &= \frac{P(TN)}{P(TN+FN)} > \frac{1}{1+\frac{P(FN)}{P(TN)}}.
\end{align*}
We can upper bound $\frac{P(FN)}{P(TN)}$ by:
\begin{equation*}
\frac{P(FN)}{P(TN)} \leq \frac{P(TP+FN)}{P(TN+FP)} \cdot \frac{F(\upsilon^{c};df,\psi^{d})}{F(\upsilon^{c};df,\psi^c)},
\end{equation*}
where $\frac{TP+FN}{TN+FP}$ denotes the ratio between the number of positive instances ($\|\theta^*_i - \theta^*_j\| > \epsilon$) and negative instances ($\|\theta^*_i - \theta^*_j\| \leq \epsilon$) in the population.  We can upper bound this ratio for any pair $(i,j)$ uniformly sampled from $[N]$, since we need to run the test on all $N^2$ pairs. First we note that $\frac{P(TP+FN)}{P(TN+FP)}=\frac{P(\lVert \theta_{i}^*-\theta_{j}^*\rVert > \epsilon)}{P(\lVert \theta_{i}^*-\theta_{j}^*\rVert \leq \epsilon)}.$ We upper-bound this ratio by giving a lower bound on the probability of two randomly sampled clients belonging to the same cluster as $P(\|\theta_i^*-\theta_j^*\| \leq \epsilon)$ with

\begin{align*}
P(\|\theta_i^*-\theta_j^*\| \leq \epsilon) &= \sum_{k=1}^M \frac{|C_k|}{N}\times\frac{|C_k|-1}{N-1}, \\
&> \sum_{k=1}^M\big(\frac{|C_k|-1}{N-1}\big)^2, \\
&> \sum_{k=1}^M\frac{1}{M^2}, \\
&= \frac{1}{M}.
\end{align*}

The second inequality is true because the probability that two uniformly sampled clients belonging to the same cluster is minimized when the clusters are all of equal sizes. Therefore we have
\begin{align*}
\frac{P(\lVert \theta_{i}^*-\theta_{j}^*\rVert > \epsilon)}{P(\lVert \theta_{i}^*-\theta_{j}^*\rVert \leq \epsilon)} \leq \frac{1-\frac{1}{M}}{\frac{1}{M}} = M-1.
\end{align*}
It is worth noting that in the event that $M=1$, the ratio can trivially be upper bounded by $1$. With this upper bound of $\frac{P(FN)}{P(TN)}$, we can now write:
\begin{equation*}
P\big(\|\theta^*_i - \theta^*_j\| \leq \epsilon|S(\cH_{t-1,i},\cH_{t-1,j}) \leq \upsilon^{c}\big) \geq {1}/\Big(1+ (M-1)\cdot \frac{F(\upsilon^{c};df,\psi^{d})}{F(\upsilon^{c};df,\psi^c)}\Big).
\end{equation*}
Then by setting $\psi^{d}=F^{-1}\big(\frac{(1-F(\upsilon^{c};df,\psi^c))}{(M-1)};df,\upsilon^{c}\big)$, we have:
\begin{equation*}
    P\big(\|\theta^*_i - \theta^*_j\|  \leq \epsilon|S(\cH_{t-1,i},\cH_{t-1,j}) \leq \upsilon^{c}\big) \geq {1}/\Big(1+ (M-1) \cdot \frac{F(\upsilon^{c};df,\psi^{d})}{F(\upsilon^{c};df,\psi^c)}\Big) = F(\upsilon^{c};df,\psi^c),
\end{equation*}
and the lemma is proven.
\end{proof}

\begin{corollary}\label{corr:type2subset}
     Under the condition that we set the threshold $\upsilon^c\geq F^{-1}(1-\frac{\delta}{N^2}, df, \psi^c)$, with an exploration phase length of $T_0= \min\{\frac{64}{3\lambda_c}\log(\frac{2Td}{\delta}), \frac{16\psi^d \sigma^2}{\lambda_c \gamma^2}\}$, we have $P(\hat{C} \subseteq C)\geq 1-\delta$.
     \begin{proof} 

     Under Assumption \ref{as:contextreg}, and with exploration length $T_0 = \min\{\frac{64}{3\lambda_c}\log(2Td / \delta), \frac{16\psi^d \sigma^2}{\lambda_c \gamma^2}\}$, the application of Lemma \ref{lem:min_eig} from \citep{Liasynch} gives with probability $1-\delta$ that 
    \begin{equation*}
        \lambda_{\min}(\bX_i^\top\bX_i)\geq \frac{\lambda_c T_0}{8} = \frac{2\psi^d\sigma^2}{\gamma^2}.
    \end{equation*}
    As a result, we can apply Lemma \ref{lem:RealNegativeOutOfPredictedNegative}, which gives
    \begin{align*}
    P\big(\|\theta^*_i - \theta^*_j\| \leq \epsilon |s(\cH_{t-1,i},\cH_{t-1,j}) \leq \upsilon^{c}\big) \geq F(\upsilon^{c};df,\psi^c).
    \end{align*}

    Using the same steps as shown in Corollary \ref{corr:type1subset}, we can see our choice of test statistic threshold $\upsilon^c\geq F^{-1}(1-\frac{\delta}{N^2}; df, \psi^c)$ results in this event occurring with probability $1-\frac{\delta}{N^2}$. Because our algorithm conducts this pairwise homogeneity test across all pairs of clients, a union bound over all $N^2$ pairwise tests proves the corollary.
     \end{proof}
\end{corollary}

The combination of Corollaries \ref{corr:type1subset} and \ref{corr:type2subset} prove that based on our choice of $\upsilon^c$ and $T_0$, $\hat{\cC} = \cC$ with probability $1-\delta$. 

\section{Proof of Lemma \ref{confidence_sets}} \label{appendix:proof:confidence_sets}
In this section, we present the complete proof of the confidence ellipsoids, following similar steps to the proof of Theorem 2 in \citep{abbasi2011improved}. 

Before we begin the proof, we will introduce a couple of useful notations to prevent clutter. Recall from Section \ref{subsec:setting} that the design matrix of client $i$, denoted as $\bX_i$, only contains the observations made by client $i$ through timestep $t$ and does not include aggregated observations from other clients. In this proof, we assume without loss of generality, that client $i$ is a member of ground-truth cluster $C_k$ and is therefore collaborating with clients $j\in C_k$. As a result, we can denote $\overline{V}_{t,i}^{-1} = \lambda I + \sum_{j\in {C}_k} \bX_j^\top \bX_j$ and $b_{t,i} = \sum_{j \in {C}_k} \mathbf{X}_{j}^\top(\mathbf{X}_j\theta_j^*+\eta_j)$ due to the sharing of sufficient statistics among clients in ${C}_k$ (line 21 in Alg. \ref{alg:simplified}), where we denote $\eta_j = (\eta_{1,j},\eta_{2,j},...,\eta_{t,j})^\top$. Note that in this proof, we only focus on the case where client $i$ is collaborating with members of its ground-truth cluster, because in Theorem \ref{thm:clustering_correctness}, we already prove with high probability $\cC = \hat{\cC}$. In our subsequent regret analysis in Theorem \ref{thm:regret_comm}, we demonstrate that the regret incurred when $\cC \neq \hat{\cC}$ is upper bounded by a constant.

\begin{proof}
    \begin{align*}
    \hat{\theta}_{t,i} &= \overline{V}_{t,i}^{-1}b_{t,i}, \\
    &= \overline{V}_{t,i}^{-1}\sum_{j \in C_k} \mathbf{X}_{j}^\top(\mathbf{X}_j\theta_j^*+\eta_j), \\
    &= \overline{V}_{t,i}^{-1}\biggr[\sum_{j \in C_k} \mathbf{X}_{j}^\top\mathbf{X}_{j}\theta_j^* + \sum_{j \in C_k}\mathbf{X}_{j}^\top\eta_j\biggr],\\
    &= \overline{V}_{t,i}^{-1}\biggr[\sum_{j \in C_k}\mathbf{X}_j^\top\mathbf{X}_j\theta_{i}^{*}+\sum_{j \in C_k\setminus \{i\}} \mathbf{X}_{j}^\top\mathbf{X}_j(\theta_{j}^{*}-\theta_{i}^{*}) + \sum_{j \in C_k}\mathbf{X}_{j}^\top\eta_j\biggr],\\
    &= \overline{V}^{-1}_{t,i}\biggr[(\lambda I + \sum_{j \in C_k}\mathbf{X}_j^\top\mathbf{X}_j)\theta_{i}^{*}-\lambda \theta_{i}^{*} +\sum_{j \in C_k\setminus \{i\}} \mathbf{X}_{j}^\top\mathbf{X}_{j}(\theta_{j}^{*}-\theta_{i}^{*}) + \sum_{j \in C_k}\mathbf{X}_{j}^\top\eta_j\biggr],\\
    &= \overline{V}_{t,i}^{-1}\overline{V}_{t,i}\theta_{i}^{*}-\lambda \overline{V}_{t,i}^{-1}\theta_{i}^{*} +\overline{V}_{t,i}^{-1}\sum_{j \in C_k\setminus \{i\}} \mathbf{X}_{j}^\top\mathbf{X}_j(\theta_{j}^{*}-\theta_{i}^{*}) + \overline{V}_{t,i}^{-1}\sum_{j \in C_k}\mathbf{X}_{j}^\top\eta_j .
    \end{align*}
As a result, we have,
\begin{equation*}
    \hat{\theta}_{t,i}-\theta_{i}^{*}= \overline{V}_{t,i}^{-1}\sum_{j \in C_k}\mathbf{X}_{j}^\top\eta_j - \lambda \overline{V}_{t,i}^{-1}\theta_{i}^{*} +\overline{V}_{t,i}^{-1}\sum_{j \in C_k\setminus \{i\}} \mathbf{X}_{j}^\top\mathbf{X}_j(\theta_{j}^{*}-\theta_{i}^{*}).
\end{equation*}

Applying the self-normalized bound gives: 
\begin{align*}
\big\|\theta_{i}^{*} - \hat{\theta}_{t,i}\big\|_{\overline{V}_{t,i}} &\le \big\|\sum_{j \in C_k}\mathbf{X}_{j}^\top\eta_j\big\|_{\overline{V}_{t,i}^{-1}} + \sqrt{\lambda}\big\|\theta_{i}^{*}\big\|_{\overline{V}_{t,i}^{-1}} +\bigg\|\sum_{j \in C_k\setminus \{i\}}^{} \mathbf{X}_{j}^\top\mathbf{X}_j(\theta_{j}^{*}-\theta_{i}^{*})\bigg\|_{\overline{V}_{t,i}^{-1}}, \\
 &\le \big\|\sum_{j \in C_k}\mathbf{X}_{j}^\top\eta_j\big\|_{\overline{V}_{t,i}^{-1}} + \sqrt{\lambda}\|\theta_{i}^{*}\|_2 +\bigg\|\sum_{j \in C_k\setminus \{i\}} \mathbf{X}_{j}^\top\mathbf{X}_j(\theta_{j}^{*}-\theta_{i}^{*})\bigg\|_{\overline{V}_{t,i}^{-1}},
\end{align*}
where we used that $\|\theta_*\|^2_{\overline{V}_{t,i}^{-1}} \leq \frac{1}{\lambda_{\min}(\overline{V}_{t,i})}\|\theta_*\|^2 \leq \frac{1}{\lambda}\|\theta_*\|^2$.

The application of Lemma \ref{lem:self_normalized_bound} and using $\|\theta_{i}^{*}\|_2 \leq 1$ give: 
\begin{align*}
& \|\theta_{i}^{*} - \hat{\theta}_{t,i}\|_{\overline{V}_{t,i}} \\
&\leq \sigma \sqrt{2 \log \biggr( \frac{\det(\overline{V}_{t,i})^{1/2} \det(\lambda I)^{-1/2}}{\delta}\biggr)} + \sqrt{\lambda} + \bigg\|\sum_{j \in C_k\setminus \{i\}}^{} \mathbf{X}_{j}^\top\mathbf{X}_j(\theta_{j}^{*}-\theta_{i}^{*})\bigg\|_{\overline{V}_{t,i}^{-1}}:= \beta_{t,i},
\end{align*}
with probability at least $1-\delta$.
Then with a union bound over all $N$ clients applied to the inequality above, we prove that $\|\theta_{i}^{*} - \hat{\theta}_{t,i}\|_{\overline{V}_{t,i}} \leq \beta_{t,i}, \forall i,t$ with probability at least $1-N\delta$.
\end{proof}

\section{Proof of Theorem \ref{thm:regret_comm}}\label{appendix:proof:regret_comm}
In this section we present the full proof of our algorithm's cumulative regret and communication upper bounds. Before proving the theorem, we will need to prove the following lemmas.

\begin{lemma}[\textbf{Heterogeneity Term Bound}] \label{lem:hetero_bound}
Under the condition that the homogeneity test threshold $\upsilon^c$ is set to be greater than $F^{-1}(1-\frac{\delta}{N^2}, df, \psi^c)$, and with an exploration phase length of $T_0 = \min\{\frac{64}{3\lambda_c}\log(\frac{2Td}{\delta}), \frac{16\psi^d \sigma^2}{\lambda_c \gamma^2}\}$ we have with probability $1-\delta$:
\begin{equation*}
\bigg\|\sum_{j \in {C}_{k}\setminus \{i\}}^{} \mathbf{X}_{j}^\top\mathbf{X}_j(\theta_{j}^{*}-\theta_{i}^{*})\bigg\|_{\overline{V}_{t,i}^{-1}}\leq 1 .
\end{equation*}
\end{lemma}

\begin{proof}
\begin{align*}
\|\sum_{j \in C_k\setminus \{i\}} \mathbf{X}_{j}^\top\mathbf{X}_j(\theta_{j}^{*}-\theta_{i}^{*})\|_{\overline{V}_{t,i}^{-1}} &\leq \sum_{j \in C_k\setminus \{i\}} \|\mathbf{X}_{j}^\top\mathbf{X}_j(\theta_{j}^{*}-\theta_{i}^{*})\|_{\overline{V}_{t,i}^{-1}}, \\ 
 &= \sum_{j \in C_k\setminus \{i\}} \sqrt{(\theta_{j}^{*}-\theta_{i}^{*})^\top \mathbf{X}_{j}^\top\mathbf{X}_j(\lambda I + \sum_{i \in C_k}\mathbf{X}_{i}^\top\mathbf{X}_i)^{-1} \mathbf{X}_{j}^\top\mathbf{X}_j (\theta_{j}^{*}-\theta_{i}^{*})},\\
  &\leq \sum_{j \in C_k\setminus \{i\}} \sqrt{(\theta_{j}^{*}-\theta_{i}^{*})^\top \mathbf{X}_{j}^\top\mathbf{X}_j(\mathbf{X}_{j}^\top\mathbf{X}_j)^{-1} \mathbf{X}_{j}^\top\mathbf{X}_j (\theta_{j}^{*}-\theta_{i}^{*})}, \\
 &\leq \sum_{j \in C_k\setminus \{i\}} \|\theta_{j}^{*}-\theta_{i}^{*}\|\sqrt{\lambda_\text{max}(\mathbf{X}_{j}^\top\mathbf{X}_j)},\\
  &\leq \sum_{j \in C_k\setminus \{i\}} \|\theta_{j}^{*}-\theta_{i}^{*}\|\sqrt{t},  \\
 &\leq \sum_{j \in C_k\setminus \{i\}} \|\theta_{j}^{*}-\theta_{i}^{*}\|\sqrt{T},
\end{align*}
where the first inequality is given by the triangle inequality. The second inequality holds because the sum over all clients $\overline{V}_{t,i} = \sum_{i \in {C_k}}\mathbf{X}_{i}^\top\mathbf{X}_i$ necessarily includes $\mathbf{X}_{j}^\top\mathbf{X}_j$, hence $\overline{V}_{t,i} \geq \mathbf{X}_{j}^\top\mathbf{X}_j$. Additionally, $\mathbf{X}_{j}^\top\mathbf{X}_j$ is positive semi-definite so that $\overline{V}_{t,i}^{-1} \leq (\mathbf{X}_{j}^\top\mathbf{X}_j)^{-1}$.  The third inequality is given by the Rayleigh-Ritz Theorem. We have the last inequality because we know that since $\|x_{t,i}\|\leq 1$, $\lambda_{\max}(\mathbf{X}_i^\top \mathbf{X}_i)\leq \tau_i \leq T$. 

Theorem \ref{thm:clustering_correctness} shows that by setting $\upsilon_c \geq F^{-1}(\frac{\delta}{N^2}, df, \psi^c)$ and $T_0 = \min\{\frac{64}{3\lambda_c}\log(\frac{2Td}{\delta}), \frac{16\psi^d \sigma^2}{\lambda_c \gamma^2}\}$, with probability $1-\delta$ we have $\hat{\cC} = \cC$. Therefore, since $i,j$ belong to the same ground-truth cluster $C_k$, we have by Assumption \ref{ass:proximity}, $\|\theta_{j}^{*}-\theta_{i}^{*}\| \leq \epsilon_t = \frac{1}{N\sqrt{t}}$. As a result, we can further upper bound the heterogeneity term by 
\begin{align*}
\bigg\|\sum_{j \in {C}_k\setminus \{i\}} \mathbf{X}_{j}^\top\mathbf{X}_j(\theta_{j}^{*}-\theta_{i}^{*})\bigg\|_{\overline{V}_{t,i}^{-1}} \leq \sum_{j \in {C}_k\setminus \{i\}} \frac{\sqrt{T}}{N\sqrt{T}} 
\leq 1.
\end{align*}
\end{proof}

\begin{lemma}
    We define the single step pseudo regret $r_{t,i}=\langle \theta_i^*, x_{t,i}^*-x_{t,i}\rangle$ where $x_{t,i}^* = \argmax_{x \in \cA_{t,i}}{\langle x,{\theta}_{t,i}^*\rangle}$. With probability $1-N\delta$, $r_{t,i}$ is bounded by
    \begin{align}
        r_{t,i} &\leq 2\biggr(\sigma\sqrt{2\log\biggr(\frac{\det(\overline{V}_{t,i})^{1/2}\det(\lambda I)^{-1/2}}{\delta}\biggr)} + \sqrt{\lambda}S + O(1) \biggr)\|x_{t,i}\|_{\overline{V}_{t,i}^{-1}} = O\biggr(\sigma\sqrt{d\log \frac{T}{\delta}}\biggr) \|x_{t,i}\|_{\overline{V}_{t,i}^{-1}} .\label{eq:singlestep_reg}
    \end{align}
\end{lemma}
    \begin{proof}
        Assume without loss of generality $\theta_{i}^* \in {C_k}$. Then,
        \begin{align}
        \nonumber
            r_{t,i} &= \langle \theta_i^*,x_{t,i}^* \rangle - \langle \theta_i^*,x_{t,i} \rangle, \\ \nonumber
            &\leq \langle \Tilde{\theta}_{t,i},x_{t,i} \rangle - \langle \theta_i^*,x_{t,i} \rangle, \\ \nonumber
            &= \langle \Tilde{\theta}_{t,i} - \theta_i^*,x_{t,i} \rangle, \\ \nonumber
            &= \langle \Tilde{\theta}_{t,i}-\hat{\theta}_{t,i},x_{t,i} \rangle + \langle \hat{\theta}_{t,i}-\theta_i^*,x_{t,i} \rangle, \\ \nonumber
            &\leq \|\Tilde{\theta}_{t,i}-\hat{\theta}_{t,i}\|_{\overline{V}_{t,i}} \|x_{t,i}\|_{\overline{V}_{t,i}^{-1}}  + \|\hat{\theta}_{t,i}-\theta_i^*\|_{\overline{V}_{t,i}} \|x_{t,i}\|_{\overline{V}_{t,i}^{-1}}, \\
            &\leq 2\biggr(\sigma\sqrt{2\log\biggr(\frac{\det(\overline{V}_{t,i})^{1/2}\det(\lambda I)^{-1/2}}{\delta}\biggr)} + \sqrt{\lambda}S + \bigg\|\sum_{j \in C_k\setminus \{i\}}^{} \mathbf{X}_{j}^\top\mathbf{X}_j(\theta_{j}^{*}-\theta_{i}^{*})\bigg\|_{\overline{V}_{t,i}^{-1}}\biggr)\|x_{t,i}\|_{\overline{V}_{t,i}^{-1}}. \label{eq:inst-reg}
        \end{align}
        The first inequality is because $\langle \Tilde{\theta}_{t,i},x_{t,i} \rangle$ is optimistic. Applying Lemma \ref{lem:hetero_bound} to upper bound the heterogeneity term gives
        \begin{align*}
            \text{RHS of Eq.\eqref{eq:inst-reg}}&\leq 2\biggr(\sigma \sqrt{2\log\biggr(\frac{\det(\overline{V}_{t,i})^{1/2}\det(\lambda I)^{-1/2}}{\delta}\biggr)} + \sqrt{\lambda}S + O(1)\biggr)\|x_{t,i}\|_{\overline{V}_{t,i}^{-1}}, \\
            &= O\biggr(\sigma\sqrt{d\log \frac{T}{\delta}}\biggr) \|x_{t,i}\|_{\overline{V}_{t,i}^{-1}} .
        \end{align*}
    \end{proof}

Now we are equipped to prove Theorem \ref{thm:regret_comm}.
\begin{proof}
    The cumulative regret of our system can be decomposed into three components. The first component is the regret accumulated under our exploration stage. During these timesteps we can trivially upper bound the instantaneous regret by two. The second component considers the regret during timesteps in which our estimated clusters are correct. 
    The third component considers the regret accumulated during the timesteps in which our estimated clusters are incorrect, which we can also upper bound the instantaneous regret by two, yielding
    \begin{equation*}
        R_T \leq \sum_{t=0}^{T_0}\sum_{i=1}^N 2 + \sum_{t=T_0+1}^T \sum_{k=1}^{\hat{M}} \sum_{i\in {\hat{C}_k}} r_{t,i} \cdot \mathds{1}\{\hat{\cC} = \cC\}   + \sum_{t=T_0+1}^T \sum_{k=1}^{\hat{M}} \sum_{i\in {\hat{C}_k}} 2 \cdot \mathds{1}\{\hat{\cC} \neq \cC\} .
    \end{equation*}
    Note that because our cluster estimation is non-parametric the number of estimated clusters $\hat{M}$ is not a hyper-parameter to our clustering algorithm. 

    According to Theorem \ref{thm:clustering_correctness}, if we select $\upsilon^c \geq F^{-1}(1-\frac{\delta}{N^2}, df, \psi^c)$ and $T_0 = \frac{16\psi^d \sigma^2}{\lambda_c \gamma^2}$, the probability that $\hat{\cC} = \cC$ is $1-\delta$. Therefore, by setting $\delta = \frac{1}{N^2 T}$, the regret contributed by the rightmost term is of the order $O(1)$. As a result, our high-probability regret bound is given by:
    \begin{equation*}
    R_T \leq \frac{32N\psi^d \sigma^2}{\lambda_c \gamma^2} + \sum_{t=T_0+1}^T \sum_{k=1}^{\hat{M}} \sum_{i\in {\hat{C}_k}} r_{t,i}\cdot \mathds{1}\{\hat{\cC} = \cC\}  + O(1).
    \end{equation*}
    
    In the subsequent steps, we will focus on the regret accumulated when $\hat{\cC} = \cC$. This means we only need to examine the instances when the estimated number of clusters and their compositions exactly match the actual clusters. Consequently, in subsequent discussions $\hat{M} = M$ and $\hat{C}_k = C_k$ for all $k \in [M]$. 

    Now we prove Theorem \ref{thm:regret_comm} following the steps in the proof of Theorem 4 in \cite{dislinucb}.
    We consider the case where Eq.\eqref{eq:singlestep_reg} holds because with the same choice of $\delta = \frac{1}{N^2 T}$, the expected instantaneous regret resulting during timesteps when Eq.\eqref{eq:singlestep_reg} does not hold is $O(1)$.
    
In our communication protocol, for each cluster $C_k$, there will be a number of epochs separated by communication rounds. We denote $|C_k|$ denotes the number of clients in cluster $C_k$. If there are $P_k$ epochs within cluster $C_k$, then $V_{P_k}$ will be the matrix with all samples from $C_k$ included. Similarly we denote the last globally shared $V$ to the clients in $C_k$ in epoch $p$ as $V_p$.

    From Lemma \ref{lem:det_trace}, we have $\det(V_0) = \lambda^d$. $\det(V_{P,k}) \leq \biggr(\frac{tr(V_{p})}{d}\biggr)^d \leq \left(\lambda + \frac{|C_k|\,T}{d}\right)^d$. Therefore by the pigeonhole principle,
    \begin{equation*}
    \log \frac{\det(V_{p})}{\det(V_0)} \leq d\log\biggr(1+\frac{|C_k|\,T}{\lambda d}\biggr).
    \end{equation*}

   It follows that for all but $R:= d\log\big(1+\frac{|C_k|T}{\lambda d}\big)$ epochs,
    \begin{equation}
        1 \leq \frac{\det(V_{j})}{\det(V_{j-1})} \leq 2. \label{eq:good_epoch}
    \end{equation}

    In these ``good epochs'' where Eq \eqref{eq:good_epoch} is satisfied, we can follow Theorem 4 from \cite{dislinucb} and treat all of the $|C_k|T$ observations from cluster $k$ as observations from an imaginary single agent in a round-robin manner. We similarly use $\tilde{V}_{t,i} = \lambda I + \sum_{\{(p,q):(p<t) \lor (p=t \land q <i)\}} x_{p,q}x_{p,q}^\top$ to denote the $\overline{V}_{t,i}$ this agent calculates before seeing $x_{t,i}$. If $x_{t,i}$ is in a good epoch, then:
    \begin{equation*}
        1 \leq \frac{\det(\tilde{V}_{t,i})}{\det(\overline{V}_{t,i})}  \leq \frac{\det(V_{j})}{\det(V_{j-1})} \leq 2. 
    \end{equation*}    
    
    We similarly denote $\mathcal{B}_{p,k}$ as the set of $(t,i)$ pairs that belong to epoch $p$ and $P_{good,k}$ as the set of good epochs in cluster $k$. In that event, we can use the regret bound for a single agent which gives

    \begin{align*}
        R_{good} &= \sum_{t=T_0+1}^T \sum_{k=1}^M \sum_{i\in {C_k}} r_{t,i},\\
                   &\leq  \sum_{k=1}^M \sqrt{|C_k|T \sum_{p \in P_{good,k}} \sum_{(t,i) \in \mathcal{B}_{p,k}}  r_{t,i}^2}, \\
                  &\leq  \sum_{k=1}^M O\bigg(\sqrt{d|C_k|T\log(T/\delta) \sum_{p \in P_{good,k}} \sum_{(t,i) \in \mathcal{B}_{p,k}}  \min(\|x_{t,i}\|^2_{\tilde{V}_{t,i}^{-1}},1)}\bigg),  \\
                &\leq  \sum_{k=1}^M O\bigg(\sqrt{d|C_k|T\log(T/\delta) \sum_{p \in P_{good,k}} \log\bigg(\frac{\det(V_p)}{\det(V_{p-1})}\bigg)}\bigg),  \\
                &\leq  \sum_{k=1}^M O\bigg(\sqrt{d|C_k|T\log(T/\delta) \log\bigg(\frac{\det(V_p)}{\det(V_{0})}\bigg)}\bigg),  \\
        &\leq \sum_{k=1}^M O\bigg( d\sqrt{|C_k|\,T}\log(|C_k|\,T)\bigg).
    \end{align*}
    Now we must analyze the regret caused by the bad epochs, of which there are $R=O(d\log(|C_k|\,T))$ within each cluster $C_k \in \cC$. This part of the analysis differs from the proof in Theorem 4 of \citep{dislinucb} due to the fact that in our protocol, clusters that have requested collaboration may have to wait in the queue until they are served in the event that multiple clusters have requested collaboration at the same timestep. 
    
Consider the regret for a particular cluster $C_k\in \cC$ during this bad epoch. Suppose that the bad epoch starts at time $t_0$ and lasts $n$ timesteps. We denote the time $t_q$ when the cluster $k$ is added to the queue awaiting collaboration. We can decompose the regret of this cluster during the bad epoch into two parts, corresponding to the timesteps before and after $C_k$ has been added to the queue:
    \begin{align*}
        REG_{bad}(k)&= \sum_{t=t_0}^{t_q-1} r_{t,i} + \sum_{t=t_q}^{n} r_{t,i}. \\
    \end{align*}
    Based on our algorithm design, we can see in line 13 of Algorithm \ref{alg:simplified} that a cluster is only added to the queue when at least one client in that cluster has exceeded its communication threshold $D_k$. Therefore we know that before $t_q$, we can upper bound the regret of the cluster $k$ from $t_0$ to $t_q-1$ as:
    \begin{align*}
        \sum_{t=t_0}^{t_q-1} r_{t,i} &\leq O\biggr(\sqrt{d\log \frac{T}{\delta}}\biggr) \sum_{i\in {C_k}} \sum_{t=t_0}^{t_q-1}  \|x_{t,i}\|_{\overline{V}_{t,i}^{-1}}, \\
        &\leq O\biggr(\sqrt{d\log \frac{T}{\delta}}\biggr) \sum_{i\in {C_k}}  \sqrt{(t_q-1 - t_0)\log \frac{\det (V_{t_q-1,i})}{\det (V_{t_q-1,i}-\Delta V_{t_q-1,i})}}, \\
        &\leq O\biggr(\sqrt{d\log \frac{T}{\delta}}\biggr) |C_k| \sqrt{D_k}.
    \end{align*}

    Once cluster $k$ is added to the queue at timestep $t_q$, it may have to wait to be served by the central server based on how many clusters have requested collaboration before it. Recall that our queue is a FIFO queue (line 19 in Algorithm \ref{alg:simplified}), and we have $M$ total clusters. Therefore the maximum time cluster $C_k$ could have to wait in the queue is $M$ timesteps. Each timestep the cluster is waiting in the queue, a client in this cluster will miss $|C_k|$ observations. For each of these missed observations, we can upper bound the regret incurred by $2$, giving
    \begin{align*}
        \sum_{t=t_q}^{n} r_{t,i} \leq 2(M+1)|C_k|^2.
    \end{align*}

    Combining our results, we have the following bound on the regret of cluster $C_k$ during a bad epoch:
    \begin{align*}
        REG_{bad}(k)&\leq O\biggr(\sqrt{d\log \frac{T}{\delta}}\biggr) |C_k| \sqrt{D_k} + 2(M+1)|C_k|^2.
    \end{align*}
    As we know we have at most $R=O(d\log(|C_k|\,T))$ bad epochs, we can further bound it by
    \begin{align*}
        REG_{bad}(k)&\leq O\biggr(\sqrt{D_k}|C_k|d^{1.5}\log^{1.5}(|C_k|T) + 2d|C_k|^2(M+1)\log(|C_k|\,T)\biggr).
    \end{align*}
    with the choice of $D_k = \frac{T \log |C_k|T}{d |C_k|}$, our regret becomes:
    \begin{align*}
        REG_{bad}(k)&\leq O\biggr(d\sqrt{|C_k|T}\log^{2}(|C_k|T) + 2d|C_k|^2M\log(|C_k|\,T)\biggr).
    \end{align*}

    The summation over all $M$ clusters gives the regret for all clusters in all of the bad epochs:
    \begin{align*}
    REG_{bad}&\leq \sum_{k=1}^M O\biggr(d\sqrt{|C_k|T}\log^{2}(|C_k|T) + 2d|C_k|^2M\log(|C_k|\,T)\biggr).
    \end{align*}    

    Combining the regret from the exploration phase, good epochs, and bad epochs gives a final cumulative regret upper bound of:
    \begin{align*}
        R_T\leq \frac{32N\psi^d \sigma^2}{\lambda_c \gamma^2} &+ \sum_{k=1}^M O\bigg(d\sqrt{|C_k|\,T}\log(|C_k|\,T)\bigg), \\ &+  \sum_{k=1}^M O\biggr(d\sqrt{|C_k|T}\log^{2}(|C_k|T) + 2d|C_k|^2M\log(|C_k|\,T)\biggr) + O(1).
    \end{align*}
    This can be further simplified with, 
    \begin{align*}
    R_T&\leq O\biggr( \frac{N\psi^d \sigma^2}{\lambda_c \gamma^2} + \sum_{k=1}^M d\sqrt{|C_k|T}\log^{2}(|C_k|T) + 2d|C_k|^2M\log(|C_k|\,T)\biggr).
    \end{align*}

\subsection{Communication cost} 

The cumulative communication cost $C_T$ of our algorithm can be divided into two parts. The first is the communication cost associated with the pure exploration and cluster estimation phase. During the pure exploration phase, no clients communicate with the central server, so that the communication cost associated with that phase is trivially $0$. At the end of the exploration phase, all $i\in [N]$ clients share with server their sufficient statistics $V_{T_0,i}$ and $b_{T_0,i}$, each of which are $d \times d$ and $d \times 1$ respectively. Therefore, the communication cost of the cluster estimation is $C_{cluster\_est} = N(d^2 + d) = O(Nd^2)$. 

Next, we characterize the communication cost of the second phase, the federated clustered bandit phase. In our communication protocol, for each cluster $C_k$, there will be a number of epochs separated by communication rounds. Denote the length of an epoch as $\alpha$, so that there can be at most $\lceil \frac{T}{\alpha}\rceil$ epochs with length longer than $\alpha$.
For an epoch with less than $\alpha$ time steps, similarly, we denote the first time step of this epoch as $t_{s}$ and the last as $t_{e}$, i.e., $t_{e}-t_{s} < \alpha$. Therefore, $\log{\frac{\det(V_{t_{e}})}{\det(V_{t_{s}})}} > \frac{D_k}{\alpha}$. Following the same argument as in the regret proof, the number of epochs with less than $\alpha$ time steps is at most $\lceil \frac{R \alpha}{D_k}\rceil$. Then $C_{fed\_cluster}(k)=|C_k| \cdot (\lceil \frac{T}{\alpha}\rceil+\lceil \frac{R \alpha}{D_k}\rceil)$, because at the end of each epoch, the synchronization round incurs $2|C_k|$ communication cost. We minimize $C_{fed\_cluster}(k)$ by choosing $\alpha=\sqrt{\frac{D_k T}{R}}$, so that $C_{fed\_cluster}(k) =O(|C_k| \cdot \sqrt{\frac{T R}{D}})$. With our choice of $D_k = \frac{T \log |C_k|T}{d |C_k|}$, we have 
\begin{align*}
    C_{fed\_cluster}(k)&= O(|C_k| \cdot \sqrt{\frac{T R}{\frac{T \log |C_k|T}{d |C_k|}}}), \\
    &= O(|C_k| \cdot d\sqrt{|C_k|}).
\end{align*}

Combining our communication cost from our two phases together gives:
\begin{align*}
    C_{T}&= O(Nd^2) + \sum_{k=1}^M O(d^3|C_k|^{1.5}).
\end{align*}
\end{proof}

\section{Additional Empirical Enhancement Evaluation} \label{appendix:empirical}
In this work, we demonstrate the effectiveness of our empirical enhancements on two synthetic datasets. In Section \ref{experiment:synthetic}, we analyzed the performance of both \model{} and \modelenhance{} on a balanced synthetic dataset that was generated following the procedure described in Section \ref{experiment:synthetic-generation}.  In this section, we evaluate our models on an imbalanced synthetic dataset to emphasize the distinct contributions of our priority queue and data-dependent re-clustering enhancements.

\paragraph{Dataset}  In this imbalanced dataset, we deliberately vary the distribution of clients and the sizes of clusters. We establish $N=50$ clients and $M=13$ ground-truth clusters. Instead of randomly assigning clients to clusters like we did in Section \ref{experiment:synthetic-generation},  we manually assigned 26 clients to cluster $C_1$, and the remaining $24$  clients were assigned in pairs to the remaining $12$ cluster centers. After being assigned to a cluster center, we follow the same procedure from Section \ref{experiment:synthetic-generation} to generate the client parameters within $\epsilon$ of the cluster centers. For the other environment settings, we used $d=25$, $K=1000$, $\gamma = 0.85$ and $T=2500$.

\paragraph{Models} In order to evaluate the contributions of each enhancement proposed in Section \ref{method:enhance}, we implemented two additional enhanced algorithms of \model{}. In HetoFedBandit-PQ, we replace the server's FIFO queue with a priority queue that selects a cluster to collaborate with based on their determinant ratios. HetoFedBandit-DR performs data-dependent clustering at each collaboration round. \modelenhance{}, as described in Algorithm \ref{alg:data-depsimplified}, is our fully enhanced algorithm, where both a priority queue and data-dependent clustering are employed.

\paragraph{Results} In Figure \ref{fig:enhance1_reg}, we conducted an empirical evaluation of the individual enhancements proposed in Section \ref{method:enhance}. A comparison between HetoFedBandit-DR and HetoFedBandit demonstrates that the use of data-dependent clustering significantly improved performance on our imbalanced synthetic dataset. By employing a data-dependent clustering threshold, our algorithm facilitated greater collaboration among clients with similar observation histories during the early rounds. Although this enhancement incurred additional communication cost, the cost remained sub-linear and comparable to that of DisLinUCB.

Comparing HetoFedBandit with HetoFedBandit-PQ, our observations suggest that the utilization of a priority queue yielded modest improvements in cumulative regret, particularly in the initial rounds when multiple clients simultaneously requested collaboration, leading to queue congestion. 
In this imbalanced environment, we observed significant delays for the larger cluster $C_1$ when using a FIFO queue. The larger size of cluster $C_1$ resulted in a higher value of the cluster determinant ratio, indicating its potential for greater regret reduction in the federated learning system. However, due to the FIFO queue, several smaller clusters that that benefited less from collaboration were served ahead of $C_1$. Nevertheless, as the algorithm progressed, the frequency of communication among clients decreased, resulting in reduced queue congestion. As a result, HetoFedBandit and HetoFedBandit-PQ exhibited similar performance in the later rounds.

\begin{figure}
     \centering
     \begin{subfigure}[b]{0.47\textwidth}
         \centering
         \includegraphics[width=\textwidth,height=5cm]{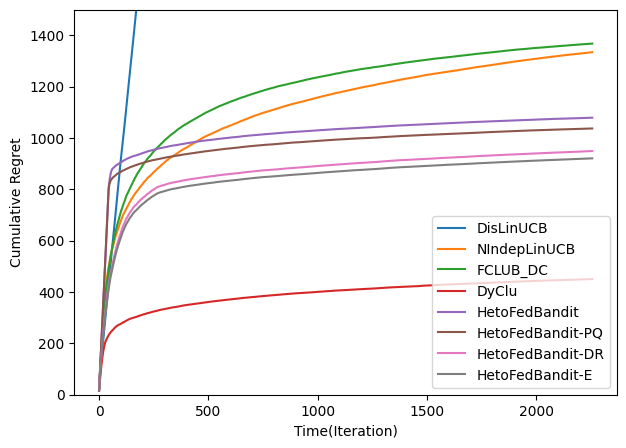}
         \caption{Accumulative Regret}
         \label{fig:enhance1_reg}
     \end{subfigure}
     \hfill
     \begin{subfigure}[b]{0.47\textwidth}
         \centering
         \includegraphics[width=\textwidth, height=5cm]{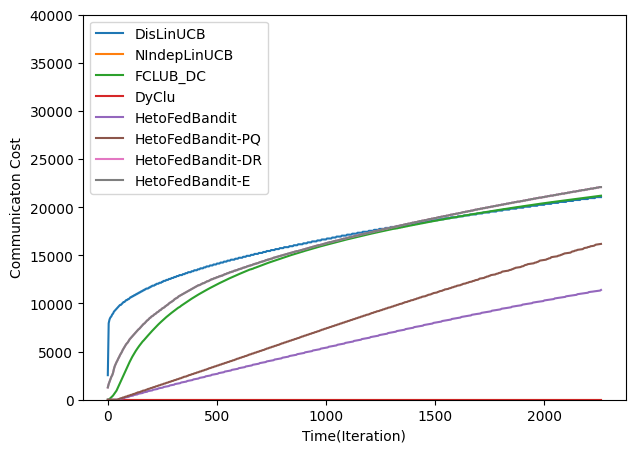}
         \caption{Communication Cost}
         \label{fig:enhance1_comm_cost}
     \end{subfigure}
     \hfill
        \caption{Experimental Results on Imbalanced Synthetic Dataset}
        \label{fig:results_sim_enhance}
        \vspace{-3mm}
\end{figure}

\section{Sensitivity Analysis} \label{appendix:sensitivity}

According to our regret analysis, the performance of \model{} depends on three key environment parameters: the number of ground-truth clusters $M$, and the number of clients $N$, and the cluster separation parameter $\gamma$. 
In this experiment, we analyze their influence on \model{} and baselines by varying these parameters while keeping the others fixed. The accumulated regret under different settings are reported in Table \ref{tb:simulationResults}. As suggested by our theoretical analysis, a larger client to cluster ratio $\frac{N}{M}$ leads to higher regret of both \model{} (HFB) and \modelenhance{} (HFB-E) as shown in setting 1, 2 and 3, since observations are split into more clusters with smaller size each. Lastly, as shown in settings 4 and 5, 6, decreasing the environment separation introduces a higher regret of \model{} since a longer exploration period is required to discern which clients are safe for collaboration. Additionally, the decreased cluster separation in setting 5 leads to an increase in regret for \modelenhance{}, as well as DyClu due to the increased likelihood of a clustering error when the clusters are closer together.

\begin{table*}[!ht]
\centering
\caption{Comparison of accumulated regret under different environment settings.}
\begin{tabular}{p{0.3cm} p{0.3cm} p{0.3cm} p{0.5cm} p{0.65cm} p{1.5cm} p{1.8cm} p{1.8cm} p{1 cm} p{0.8 cm} p{1.2cm}}
\toprule
& $N$ & $M$ & $\gamma$ & $T$ & NIndep-LinUCB & DisLinUCB & FCLUB\_DC & DyClu & HFB & HFB-E \\
\cmidrule(r{4pt}){2-5} \cmidrule(l){5-11}
1 & 30 & 1 & 0.85 & 3000 &772.03 & 59.20 & 648.22 & 48.77 & 576.31 & 173.89  \\
2 & 30 & 4 & 0.85 & 3000 &784.80 & 23776.10 & 784.01 & 227.18 & 669.17 &  443.89  \\
3 & 30 & 30 & 0.85 & 3000 &781.35 & 25124.46 & 791.19 & 776.86 & 883.51 & 822.24 \\
4 & 30 & 4 & 0.65 & 3000 &777.73 & 20129.05 & 788.12 & 231.57 & 699.89 &  461.54  \\
5 & 30 & 4 & 0.05 & 3000 & 787.79 & 23823.61 & 771.04 & 269.45 & 916.73 & 582.21  \\
\bottomrule
\end{tabular}
\label{tb:simulationResults}
\end{table*}


\vskip 0.2in
\newpage

\bibliography{bibfile}

\begin{thebibliography}{38}
\providecommand{\natexlab}[1]{#1}
\providecommand{\url}[1]{\texttt{#1}}
\expandafter\ifx\csname urlstyle\endcsname\relax
  \providecommand{\doi}[1]{doi: #1}\else
  \providecommand{\doi}{doi: \begingroup \urlstyle{rm}\Url}\fi

\bibitem[Abbasi-Yadkori et~al.(2011)Abbasi-Yadkori, P{\'a}l, and Szepesv{\'a}ri]{abbasi2011improved}
Yasin Abbasi-Yadkori, D{\'a}vid P{\'a}l, and Csaba Szepesv{\'a}ri.
\newblock Improved algorithms for linear stochastic bandits.
\newblock In \emph{Advances in Neural Information Processing Systems}, pages 2312--2320, 2011.

\bibitem[Auer et~al.(1995)Auer, Cesa-Bianchi, Freund, and Schapire]{aueradversarial1995}
P.~Auer, N.~Cesa-Bianchi, Y.~Freund, and R.E. Schapire.
\newblock Gambling in a rigged casino: The adversarial multi-armed bandit problem.
\newblock In \emph{Proceedings of IEEE 36th Annual Foundations of Computer Science}, pages 322--331, 1995.
\newblock \doi{10.1109/SFCS.1995.492488}.

\bibitem[Auer et~al.(2002)Auer, Cesa-Bianchi, and Fischer]{auer2002finite}
Peter Auer, Nicolo Cesa-Bianchi, and Paul Fischer.
\newblock Finite-time analysis of the multiarmed bandit problem.
\newblock \emph{Machine learning}, 47\penalty0 (2-3):\penalty0 235--256, 2002.

\bibitem[Besson and Kaufmann(2018)]{besson2018doubling}
Lilian Besson and Emilie Kaufmann.
\newblock What doubling tricks can and can't do for multi-armed bandits, 2018.

\bibitem[Bonawitz et~al.(2019)Bonawitz, Eichner, Grieskamp, Huba, Ingerman, Ivanov, Kiddon, Kone{\v{c}}n{\`y}, Mazzocchi, McMahan, et~al.]{bonawitz2019towards}
Keith Bonawitz, Hubert Eichner, Wolfgang Grieskamp, Dzmitry Huba, Alex Ingerman, Vladimir Ivanov, Chloe Kiddon, Jakub Kone{\v{c}}n{\`y}, Stefano Mazzocchi, Brendan McMahan, et~al.
\newblock Towards federated learning at scale: System design.
\newblock \emph{Proceedings of machine learning and systems}, 1:\penalty0 374--388, 2019.

\bibitem[Buccapatnam et~al.(2013)Buccapatnam, Eryilmaz, and Shroff]{buccapatnam2013multi}
Swapna Buccapatnam, Atilla Eryilmaz, and Ness~B Shroff.
\newblock Multi-armed bandits in the presence of side observations in social networks.
\newblock In \emph{52nd IEEE Conference on Decision and Control}, pages 7309--7314. IEEE, 2013.

\bibitem[Cantador et~al.(2011)Cantador, Brusilovsky, and Kuflik]{Cantador:RecSys2011}
Iv\'{a}n Cantador, Peter Brusilovsky, and Tsvi Kuflik.
\newblock 2nd workshop on information heterogeneity and fusion in recommender systems (hetrec 2011).
\newblock In \emph{Proceedings of the 5th ACM conference on Recommender systems}, RecSys 2011, New York, NY, USA, 2011. ACM.

\bibitem[Cantrell et~al.(1991)Cantrell, Burrows, and Vuong]{cantrell1991interpretation}
R~Stephen Cantrell, Peter~M Burrows, and Quang~H Vuong.
\newblock Interpretation and use of generalized chow tests.
\newblock \emph{International Economic Review}, pages 725--741, 1991.

\bibitem[Caron et~al.(2012)Caron, Kveton, Lelarge, and Bhagat]{kvetonsideobservations2012}
St{\'{e}}phane Caron, Branislav Kveton, Marc Lelarge, and Smriti Bhagat.
\newblock Leveraging side observations in stochastic bandits.
\newblock \emph{CoRR}, abs/1210.4839, 2012.
\newblock URL \url{http://arxiv.org/abs/1210.4839}.

\bibitem[Cesa-Bianchi et~al.(2013)Cesa-Bianchi, Gentile, and Zappella]{cesa2013gang}
Nicolo Cesa-Bianchi, Claudio Gentile, and Giovanni Zappella.
\newblock A gang of bandits.
\newblock In \emph{Advances in Neural Information Processing Systems}, pages 737--745, 2013.

\bibitem[Chapelle and Li(2011)]{chapelle2011empirical}
Olivier Chapelle and Lihong Li.
\newblock An empirical evaluation of thompson sampling.
\newblock \emph{Advances in neural information processing systems}, 24, 2011.

\bibitem[Chow(1960)]{chow1960tests}
Gregory~C Chow.
\newblock Tests of equality between sets of coefficients in two linear regressions.
\newblock \emph{Econometrica: Journal of the Econometric Society}, pages 591--605, 1960.

\bibitem[Dubey and Pentland(2020{\natexlab{a}})]{dubey2020}
Abhimanyu Dubey and Alex \textasciigrave~Sandy\textquotesingle Pentland.
\newblock Differentially-private federated linear bandits.
\newblock In H.~Larochelle, M.~Ranzato, R.~Hadsell, M.F. Balcan, and H.~Lin, editors, \emph{Advances in Neural Information Processing Systems}, volume~33, pages 6003--6014. Curran Associates, Inc., 2020{\natexlab{a}}.
\newblock URL \url{https://proceedings.neurips.cc/paper_files/paper/2020/file/4311359ed4969e8401880e3c1836fbe1-Paper.pdf}.

\bibitem[Dubey and Pentland(2020{\natexlab{b}})]{dubey2020differentially}
Abhimanyu Dubey and AlexSandy' Pentland.
\newblock Differentially-private federated linear bandits.
\newblock \emph{Advances in Neural Information Processing Systems}, 33:\penalty0 6003--6014, 2020{\natexlab{b}}.

\bibitem[Durand et~al.(2018)Durand, Achilleos, Iacovides, Strati, Mitsis, and Pineau]{durand2018contextual}
Audrey Durand, Charis Achilleos, Demetris Iacovides, Katerina Strati, Georgios~D Mitsis, and Joelle Pineau.
\newblock Contextual bandits for adapting treatment in a mouse model of de novo carcinogenesis.
\newblock In \emph{Machine learning for healthcare conference}, pages 67--82. PMLR, 2018.

\bibitem[Foley et~al.(2022)Foley, Sheller, Edwards, Pati, Riviera, Sharma, Moorthy, Wang, Martin, Mirhaji, Shah, and Bakas]{openfl_citation}
Patrick Foley, Micah~J Sheller, Brandon Edwards, Sarthak Pati, Walter Riviera, Mansi Sharma, Prakash~Narayana Moorthy, Shi-han Wang, Jason Martin, Parsa Mirhaji, Prashant Shah, and Spyridon Bakas.
\newblock Openfl: the open federated learning library.
\newblock \emph{Physics in Medicine \& Biology}, 2022.
\newblock \doi{10.1088/1361-6560/ac97d9}.
\newblock URL \url{http://iopscience.iop.org/article/10.1088/1361-6560/ac97d9}.

\bibitem[Gentile et~al.(2014)Gentile, Li, and Zappella]{gentile2014online}
Claudio Gentile, Shuai Li, and Giovanni Zappella.
\newblock Online clustering of bandits.
\newblock In \emph{International Conference on Machine Learning}, pages 757--765, 2014.

\bibitem[Gentile et~al.(2017)Gentile, Li, Kar, Karatzoglou, Zappella, and Etrue]{gentile2017context}
Claudio Gentile, Shuai Li, Purushottam Kar, Alexandros Karatzoglou, Giovanni Zappella, and Evans Etrue.
\newblock On context-dependent clustering of bandits.
\newblock In \emph{Proceedings of the 34th International Conference on Machine Learning-Volume 70}, pages 1253--1262. JMLR. org, 2017.

\bibitem[He et~al.(2020)He, Li, So, Zhang, Wang, Wang, Vepakomma, Singh, Qiu, Shen, Zhao, Kang, Liu, Raskar, Yang, Annavaram, and Avestimehr]{chaoyanghe2020fedml}
Chaoyang He, Songze Li, Jinhyun So, Mi~Zhang, Hongyi Wang, Xiaoyang Wang, Praneeth Vepakomma, Abhishek Singh, Hang Qiu, Li~Shen, Peilin Zhao, Yan Kang, Yang Liu, Ramesh Raskar, Qiang Yang, Murali Annavaram, and Salman Avestimehr.
\newblock Fedml: A research library and benchmark for federated machine learning.
\newblock \emph{Advances in Neural Information Processing Systems, Best Paper Award at Federate Learning Workshop}, 2020.

\bibitem[He et~al.(2022)He, Wang, Min, and Gu]{he2022a}
Jiafan He, Tianhao Wang, Yifei Min, and Quanquan Gu.
\newblock A simple and provably efficient algorithm for asynchronous federated contextual linear bandits.
\newblock In Alice~H. Oh, Alekh Agarwal, Danielle Belgrave, and Kyunghyun Cho, editors, \emph{Advances in Neural Information Processing Systems}, 2022.
\newblock URL \url{https://openreview.net/forum?id=Fx7oXUVEPW}.

\bibitem[Hong et~al.(2021)Hong, Kveton, Zaheer, and Ghavamzadeh]{hongbayesian}
Joey Hong, Branislav Kveton, Manzil Zaheer, and Mohammad Ghavamzadeh.
\newblock Hierarchical bayesian bandits.
\newblock \emph{CoRR}, abs/2111.06929, 2021.
\newblock URL \url{https://arxiv.org/abs/2111.06929}.

\bibitem[Hossain et~al.(2021)Hossain, Micha, and Shah]{hossain2021fair}
Safwan Hossain, Evi Micha, and Nisarg Shah.
\newblock Fair algorithms for multi-agent multi-armed bandits.
\newblock In A.~Beygelzimer, Y.~Dauphin, P.~Liang, and J.~Wortman Vaughan, editors, \emph{Advances in Neural Information Processing Systems}, 2021.
\newblock URL \url{https://openreview.net/forum?id=AlD5WD2ANIQ}.

\bibitem[Huang et~al.(2021)Huang, Wu, Yang, and Shen]{huangfedcontextualbandits}
Ruiquan Huang, Weiqiang Wu, Jing Yang, and Cong Shen.
\newblock Federated linear contextual bandits, 2021.
\newblock URL \url{https://arxiv.org/abs/2110.14177}.

\bibitem[Kairouz et~al.(2021)Kairouz, McMahan, Avent, Bellet, Bennis, Bhagoji, Bonawitz, Charles, Cormode, Cummings, et~al.]{kairouz2021advances}
Peter Kairouz, H~Brendan McMahan, Brendan Avent, Aur{\'e}lien Bellet, Mehdi Bennis, Arjun~Nitin Bhagoji, Kallista Bonawitz, Zachary Charles, Graham Cormode, Rachel Cummings, et~al.
\newblock Advances and open problems in federated learning.
\newblock \emph{Foundations and Trends{\textregistered} in Machine Learning}, 14\penalty0 (1--2):\penalty0 1--210, 2021.

\bibitem[Korda et~al.(2016)Korda, Szorenyi, and Li]{korda2016_peer_clustering}
Nathan Korda, Balazs Szorenyi, and Shuai Li.
\newblock Distributed clustering of linear bandits in peer to peer networks.
\newblock In Maria~Florina Balcan and Kilian~Q. Weinberger, editors, \emph{Proceedings of The 33rd International Conference on Machine Learning}, volume~48 of \emph{Proceedings of Machine Learning Research}, pages 1301--1309, New York, New York, USA, 20--22 Jun 2016. PMLR.
\newblock URL \url{https://proceedings.mlr.press/v48/korda16.html}.

\bibitem[Li and Wang(2022)]{Liasynch}
Chuanhao Li and Hongning Wang.
\newblock Asynchronous upper confidence bound algorithms for federated linear bandits.
\newblock In Gustau Camps-Valls, Francisco J.~R. Ruiz, and Isabel Valera, editors, \emph{Proceedings of The 25th International Conference on Artificial Intelligence and Statistics}, volume 151 of \emph{Proceedings of Machine Learning Research}, pages 6529--6553. PMLR, 28--30 Mar 2022.
\newblock URL \url{https://proceedings.mlr.press/v151/li22e.html}.

\bibitem[Li et~al.(2021)Li, Wu, and Wang]{lidyclu}
Chuanhao Li, Qingyun Wu, and Hongning Wang.
\newblock Unifying clustered and non-stationary bandits.
\newblock In \emph{International Conference on Artificial Intelligence and Statistics}, pages 1063--1071. PMLR, 2021.

\bibitem[Li et~al.(2010{\natexlab{a}})Li, Chu, Langford, and Schapire]{lihongcontextual}
Lihong Li, Wei Chu, John Langford, and Robert~E. Schapire.
\newblock A contextual-bandit approach to personalized news article recommendation.
\newblock In \emph{Proceedings of the 19th International Conference on World Wide Web}, WWW '10, page 661–670, New York, NY, USA, 2010{\natexlab{a}}. Association for Computing Machinery.
\newblock ISBN 9781605587998.
\newblock \doi{10.1145/1772690.1772758}.
\newblock URL \url{https://doi.org/10.1145/1772690.1772758}.

\bibitem[Li et~al.(2016)Li, Karatzoglou, and Gentile]{li2016collaborative}
Shuai Li, Alexandros Karatzoglou, and Claudio Gentile.
\newblock Collaborative filtering bandits.
\newblock In \emph{Proceedings of the 39th International ACM SIGIR conference on Research and Development in Information Retrieval}, pages 539--548. ACM, 2016.

\bibitem[Li et~al.(2019)Li, Chen, and Leung]{li2019improved}
Shuai Li, Wei Chen, and Kwong-Sak Leung.
\newblock Improved algorithm on online clustering of bandits.
\newblock \emph{arXiv preprint arXiv:1902.09162}, 2019.

\bibitem[Li et~al.(2010{\natexlab{b}})Li, Wang, Zhang, Cui, Mao, and Jin]{li2010exploitation}
Wei Li, Xuerui Wang, Ruofei Zhang, Ying Cui, Jianchang Mao, and Rong Jin.
\newblock Exploitation and exploration in a performance based contextual advertising system.
\newblock In \emph{Proceedings of the 16th ACM SIGKDD international conference on Knowledge discovery and data mining}, pages 27--36, 2010{\natexlab{b}}.

\bibitem[Liu et~al.(2022)Liu, Zhao, Yu, Li, and Lui]{liufederatedonlineclusteringofbandits}
Xutong Liu, Haoru Zhao, Tong Yu, Shuai Li, and John C.~S. Lui.
\newblock Federated online clustering of bandits, 2022.
\newblock URL \url{https://arxiv.org/abs/2208.14865}.

\bibitem[Mahadik et~al.(2020)Mahadik, Wu, Li, and Sabne]{mahadik2020fast}
Kanak Mahadik, Qingyun Wu, Shuai Li, and Amit Sabne.
\newblock Fast distributed bandits for online recommendation systems.
\newblock In \emph{Proceedings of the 34th ACM international conference on supercomputing}, pages 1--13, 2020.

\bibitem[Mannor and Shamir(2011)]{mannorsideobservations}
Shie Mannor and Ohad Shamir.
\newblock From bandits to experts: On the value of side-observations.
\newblock In J.~Shawe-Taylor, R.~Zemel, P.~Bartlett, F.~Pereira, and K.Q. Weinberger, editors, \emph{Advances in Neural Information Processing Systems}, volume~24. Curran Associates, Inc., 2011.
\newblock URL \url{https://proceedings.neurips.cc/paper_files/paper/2011/file/e1e32e235eee1f970470a3a6658dfdd5-Paper.pdf}.

\bibitem[McMahan et~al.(2016)McMahan, Moore, Ramage, and y~Arcas]{McMahanMRA16}
H.~Brendan McMahan, Eider Moore, Daniel Ramage, and Blaise~Ag{\"{u}}era y~Arcas.
\newblock Federated learning of deep networks using model averaging.
\newblock \emph{CoRR}, abs/1602.05629, 2016.
\newblock URL \url{http://arxiv.org/abs/1602.05629}.

\bibitem[Shi and Shen(2021)]{fedmab}
Chengshuai Shi and Cong Shen.
\newblock Federated multi-armed bandits.
\newblock \emph{CoRR}, abs/2101.12204, 2021.
\newblock URL \url{https://arxiv.org/abs/2101.12204}.

\bibitem[Wang et~al.(2020)Wang, Hu, Chen, and Wang]{dislinucb}
Yuanhao Wang, Jiachen Hu, Xiaoyu Chen, and Liwei Wang.
\newblock Distributed bandit learning: Near-optimal regret with efficient communication.
\newblock In \emph{8th International Conference on Learning Representations, {ICLR} 2020, Addis Ababa, Ethiopia, April 26-30, 2020}. OpenReview.net, 2020.
\newblock URL \url{https://openreview.net/forum?id=SJxZnR4YvB}.

\bibitem[Wu et~al.(2016)Wu, Wang, Gu, and Wang]{wu2016contextual}
Qingyun Wu, Huazheng Wang, Quanquan Gu, and Hongning Wang.
\newblock Contextual bandits in a collaborative environment.
\newblock In \emph{Proceedings of the 39th International ACM SIGIR conference on Research and Development in Information Retrieval}, pages 529--538. ACM, 2016.

\end{thebibliography}

\end{document}